\documentclass{article}

\usepackage[main, final]{neurips_2026}

\usepackage[utf8]{inputenc} 
\usepackage[T1]{fontenc}    
\usepackage{hyperref}       
\usepackage{url}            
\usepackage{booktabs}       
\usepackage{amsfonts}       
\usepackage{nicefrac}       
\usepackage{microtype}      
\usepackage{xcolor}         

\usepackage{microtype}
\usepackage{graphicx}
\usepackage{subcaption}
\usepackage{booktabs} 
\usepackage{multirow}
\usepackage{graphicx}
\usepackage{caption}
\usepackage{subcaption}
\usepackage{amsmath}
\usepackage{amssymb}
\usepackage{mathtools}
\usepackage{amsthm}
\usepackage{listings}
\usepackage[dvipsnames]{xcolor}
\usepackage[capitalize,noabbrev]{cleveref}

\theoremstyle{plain}
\newtheorem{theorem}{Theorem}[section]

\newtheorem{lemma}[theorem]{Lemma}

\theoremstyle{definition}
\newtheorem{definition}[theorem]{Definition}

\theoremstyle{remark}

\title{On the Runway Cascade of Transformers for Language Modeling}

\author{%
  Hunjae Lee \\
  Department of Computer Science\\
  Southern Methodist University\\
  Dallas, TX 75205 \\
  \texttt{hunjael@smu.edu} \\
  \And
  Corey Clark \\
  Department of Computer Science \\
  SMU Guildhall \\
  Southern Methodist University \\
  Dallas, TX 75205 \\
  \texttt{coreyc@smu.edu} \\
}

\begin{document}

\maketitle

\begin{abstract}
In decoder-only (causal) transformers, the computation graph created by causal masking routes information through both direct-path attention and indirect paths formed by intermediate tokens. We denote these indirect paths between token pairs as their \textit{runways}. We introduce \textit{runway cascade} as a phenomenon whereby a misalignment between these two information propagation modes results in redundancies and irrelevant information cascading to token representations even with adequately learned attention patterns. As a solution, we propose \textit{runway-aware rewiring} as a more explicit way of incorporating runway context directly into each token's direct-path attention. This mechanism re-wires the attention pattern for each token based on a summary of its runway landscape, enabling awareness of accumulating representational influences and allowing for more balanced information propagation. Our proposed methodology introduces no additional parameters and can seamlessly be integrated into standard attention mechanism. Empirically, we observe consistent improvements in general language modeling, extrapolation, information retrieval, and more. Code is available at \url{https://gitlab.com/humin-game-lab/rewiring-redundancy-cascades}.
\end{abstract}

\section{Introduction}
Causal, decoder-only, transformers make up a majority of modern LLM architectures \cite{achiam2023gpt,comanici2025gemini}. Despite their tremendous success, they have been observed to exhibit failure points and unexpected phenomena. Some of these instances include lost-in-the-middle phenomenon \cite{liu2024lost} on retrieval tasks, last-token representation collapse \cite{barbero2024transformers}, the emergence of position bias favoring early positions in the sequence \cite{wu2025on}, and the development of attention sinks \cite{xiao2024efficient,gu2025when}. Viewing transformers as an instance of graph neural networks (GNNs) \cite{joshi2025transformers}, many of these phenomena can be contextualized through information propagation theory of GNNs, a field with a rich literature studying representational and topological bottlenecks. Namely, theory of over-squashing and over-smoothing in GNNs have been instrumental in fostering a unified perspective on these seemingly disconnected failure modes of transformers \cite{wu2024role,barbero2024transformers,arroyo2025bridging}. 

At the heart of these failure modes lie the computation graph structure created by causal masking in decoder-only transformers. This causal graph has been shown to create position bias that, when taken to the limit, make early tokens in the sequence dominate the representation landscape \cite{wu2025on,barbero2025why}. Likewise, the causal graph structure is also responsible for creating last-token representation collapse, whereby certain distinct sequences become arbitrarily similar in their last token representation \cite{barbero2024transformers}.

In causal attention, tokens influence subsequent positions in two ways: (1) through direct connections and (2) indirectly by influencing intermediate tokens along the \textit{runway}. We define the runway of two tokens as the collection of indirect paths between them made up of intermediate tokens. To study their influence on information propagation, we formulate a GNN-theory style bound for causal transformers and show that their representational sensitivity is controlled by the given destination token's own attention pattern (direct influence) acting as a gate, composing the mixture of accumulated attention patterns along its runway (indirect influences).

In this work, we introduce \textit{runway cascade} as a phenomenon whereby direct-path attention no longer adequately controls the indirect influences and thereby enabling a cascade of redundancies and irrelevant information on token representations. We argue that this stems from standard attention mechanism remaining unaware of accumulating runway influences and show that it can occur even when the attention patterns are adequately learned. As such, we propose a more explicit way of incorporating runway context directly into row-wise attention calculation. We introduce a runway-aware rewiring technique that re-wires the attention pattern for each token based on a computed \textit{runway-coefficient}. This allows each token's attention pattern to discern accumulated representational influences along the runway and to re-weight the attention scores accordingly for more balanced information propagation. Our proposed solution introduces no additional parameters and can be seamlessly integrated with standard attention.

\section{Background and Related Works}
\subsection{Interpreting Causal Transformers as a GNN}
In order to study transformers in part through the lens of GNN theory, we first make the connection that transformers can be seen as an instance of GNNs \cite{joshi2025transformers}. Extending from \cite{joshi2025transformers} which constrained its analysis to fully-connected, non-causal attention, we interpret causal transformers as an attention-based message-passing GNN (refer to \Cref{sec:app_gnn} for details). 

\subsection{Information Propagation Theory}
Analyzing failure modes and unexpected phenomena of transformer-based language models through the lens of information propagation theory is a recent but increasingly favored approach. In particular, theory of over-squashing in GNNs has been helpful in understanding related issues in transformers.

\paragraph{Over-squashing in GNNs.}
Over-squashing \cite{alon2021on,topping2022understanding,di2023over} is a phenomenon in GNNs where exponentially growing receptive field of nodes result in over-compression of information into fixed-sized node representations, inevitably causing distortion or even loss of certain information. This has been observed as a consequence of long-range information propagation \cite{alon2021on} but also of graph topology \cite{topping2022understanding,black2023understanding,giovanni2024how}. In particular, \cite{giovanni2024how} interprets over-squashing as a result of misalignment between the graph topology and the task at hand. Many issues of transformers arising from the causal graph can be interpreted under this framework, where the causal graph creates misalignment between its topology and the task at hand despite being a practical tool for handling next-token prediction schemes. Over-squashing in GNNs is related to representation collapse in transformers, a phenomenon that has been shown to disproportionately affect the last token in the sequence \cite{barbero2024transformers}.


\paragraph{Graph Rewiring.}
A well-studied methodology to alleviate information propagation issues and to better align the graph topology to the task at hand is graph rewiring. There's broadly two different objectives: spatial rewiring \cite{abboud2022shortest,bruel2022rewiring}, and spectral rewiring \cite{karhadkar2023fosr,jamadandi2024spectral}. Many works also optimize for more specific quantities, such as effective resistance \cite{black2023understanding}, commute time \cite{giovanni2024how,sterner2024commute}, and the spectral gap of the graph Laplacian \cite{deac2022expander,wilson2024cayley}.

Graph rewiring approaches have also been applied to transformers. For the most part, they rewire the computation graph for sparsity and computational efficiency \cite{shirzad2023exphormer,shirzad2024even} and not for topological improvements. Under this interpretation, sparse transformers \cite{beltagy2020longformer,zaheer2020big} in general can be seen as ways of rewiring the attention pattern to achieve computational efficiency. 

A majority of graph rewiring approaches (including aforementioned works) opt for \textit{hard-rewiring}, whereby edges are added or deleted to change the connectivity of the graph in a binary manner. In our proposed methodology, however, we opt for a \textit{soft-rewiring} approach, changing attention patterns based on the rewiring criteria without changing the graph topology itself. We do this to maintain causality and to make the comparison against standard causal transformers as one-to-one as possible.

\subsection{Position Bias and Information Propagation Issues in Transformers}
While being a practical and effective solution for enforcing causality, the application of causal masking in the shape of an upper-triangular mask creates interesting, and potentially unexpected phenomena for transformers in practice. Namely, the causal graph of the transformer, now directed and bipartite, creates a position bias that favors early positions \cite{wu2025on}. Intuitively, the earlier the position in the sequence, the more future tokens attend back to it, giving it more opportunities to influence other tokens to preserve and refine its own information. Likewise, recent work has shown evidence that language models may plan ahead for future tokens with a technique called pre-caching \cite{wu2024language}. This has practical consequences, with works showing asymmetric performance degradation based on position dependencies or order of examples \cite{lu2022fantastically,hou2024large,fang2025rethinking}. Position bias is also related to the emergence of attention sinks \cite{xiao2024efficient,gu2025when,barbero2025why}, a phenomenon where the first few tokens of the sequence get disproportionately high attention weights. In addition, position bias of causal transformers can induce over-squashing or representation collapse of the last token in certain distinct sequences, as illustrated by \cite{barbero2024transformers}.

While some recent works have begun relating these emergent properties of causal graphs in transformers, they typically focus on their influence on absolute positions in the sequence. For example, \cite{wu2025on,barbero2025why} focus on the role of the early positions in the sequence while \cite{barbero2024transformers} studies the effects of the causal graph on the final token in the sequence. As such, a generalized understanding on how relative distances between arbitrary token pairs affect information propagation has not been made explicit. Likewise, practical solutions that address these information propagation issues created by the geometry of the causal graph remain an open problem.

Several recent works reshape causal softmax attention by down-weighting attention scores in a data-dependent manner. The Forgetting Transformer \cite{lin2025forgetting} adds a scalar forget gate to each token. The attention between two tokens is scaled by the product of the gates of the tokens between them. Selective attention \cite{leviathan2025selective} reuses an existing attention head so that tokens that have attended to an earlier token accumulate a penalty, reducing the attention all future tokens pay to it. Our runway-aware rewiring belongs to the same family. Like selective attention, it is also parameter-free and reuses an existing head. Our approach differs in that it does not accumulate. The forgetting transformer applies a decay that does not depend on the source token and only grows with distance, and selective attention's masking is permanent once applied. Runway-aware rewiring instead re-evaluates each source token for every destination by comparing it with the destination's immediate predecessor, so a token suppressed for one destination can regain weight for a later one. It also measures redundancy through similarity between latent token representations rather than through attention logits.

\paragraph{On Positional Encodings.}
The study of position bias is related, but distinct, from the study of positional encodings (PEs) \cite{kazemnejad2023the,su2024roformer,barbero2025round}. Intuitively, while PEs are a study of how order is represented in attention mechanisms, position bias is a study of how the given architecture processes order. The latter reveals emergent topological bottlenecks hindering information propagation, and how they can override inductive biases meant to be provided by the former.

\section{Runway Cascade}

\subsection{Information Flow in Transformers}\label{sec:runway_info_flow}
In causal attention, information from a given source token $h_s$ flows to its destination token $h_d$ in two ways (for arbitrary source and destination tokens $s \le d$ in the sequence). First, \textit{direct connections} allow $h_s$ to reach $h_d$ directly in one-hop connection. The strength of such direct connection is controlled by the destination token's own row-wise attention pattern. Second, \textit{indirect connections} between $h_s$ and $h_d$ allow $h_s$ to influence intermediate tokens $h_j$ (where $s < j < d$) to propagate its own information along the indirect paths to $h_d$. We denote the collection of all indirect paths between $h_s$ and $h_d$ as their \textit{runway}. 

\begin{definition}[Runway]
    Let $\mathcal{G} = (V, E)$ be a directed, bipartite graph where $V$ denotes vertices (nodes) and $E$ denotes edges. A path from nodes $s$ and $d$ is a vertex sequence $p = (v_0,...,v_l)$ with $v_0 = s$ and $v_l = d$, s.t. $(v_{i-1}, v_i) \in E$ for all $i$ and distinct vertices. An indirect path is any path where the path length $l \ge 2$, or equivalently if the internal vertex set $\{v_1,...,v_{l-1}\} \neq \varnothing$. Let $P_{\mathcal{G}}(s,d)$ be the set of all paths between $s$ and $d$. Then, the runway of nodes $s$ and $d$ is defined as the set $\mathcal{R}_{ds}$ where:
    \begin{equation}
        \mathcal{R}_{ds} = \{(v_0,...,v_l) \in P_{G}(s,d) : l \ge 2\}\nonumber
    \end{equation}
\end{definition}

\begin{figure}[t] 
  \centering
  \includegraphics[width=0.5\linewidth]{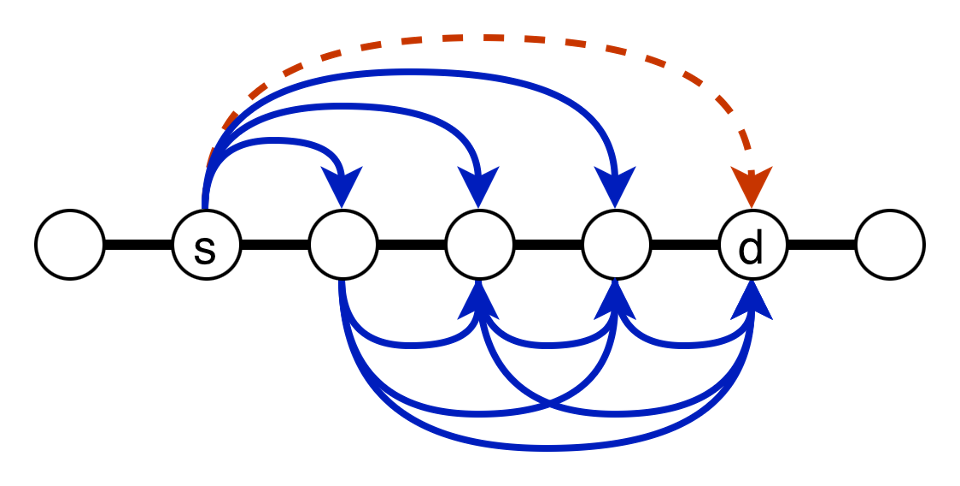}
  \caption{Visualizing two types of paths between token $s$ and $d$ along the causal graph: (1) direct path (in dotted red) and (2) indirect paths (in solid blue).}
  \label{fig:direct_vs_indrect}
\end{figure}

\Cref{fig:direct_vs_indrect} illustrates the runway $\mathcal{R}_{ds}$ of tokens $h_s$ and $h_d$ in solid blue paths while denoting the direct path in dotted red. While information of token $h_s$ has many paths flowing into $h_d$ (red + blue paths), only its direct path (in dotted red) is controlled explicitly by $h_d$ with respect to $h_s$ in isolation. On the other hand, accumulated information of $h_s$ in the runway (propagated along solid blue paths) are absorbed into intermediate tokens. While these intermediate tokens ultimately influence $h_d$'s representation through their own direct connections, $h_d$ can only attend to them without being able to isolate the contribution of $h_s$. In addition, though the intermediate tokens do have some control over whether to absorb information from $h_s$ or not in their direct connections to $h_s$, the continuous nature of the softmax function in the attention mechanism prevents them from ignoring $h_s$ entirely \cite{wu2023demystifying,wu2025on}.

\begin{lemma}[Lemma 2 from \cite{wu2023demystifying}]\label[lemma]{lem:softmax_wu}
    Assume that the attention function is continuous and the point-wise nonlinear activation function $\sigma(\cdot)$ satisfies $0 \le \frac{\sigma(x)}{x} \le 1$ for $x \neq 0$ and $\sigma(0) = 0$. Then, there exists $\epsilon > 0$ s.t. for all $t \in \mathbb{N}_{\ge 0}$ and for any $(i,j) \in E$, we have $a^{(t)}_{ij} > \epsilon$.
\end{lemma}

\Cref{lem:softmax_wu} states that attention mechanism cannot disconnect any valid edge between nodes/tokens given the continuous nature of the softmax function and that each edge $(i,j)$ must be assigned a non-zero attention weight $a^{(t)}_{ij} > \epsilon, \; \epsilon > 0$. This means that information of $h_s$ accumulates \textit{inevitably} along the runway to $h_d$. Position bias exacerbates this effect, giving $h_s$ more opportunities to preserve and refine its information the earlier its position in the sequence. Because the only path through which $h_d$ can control $h_s$ in isolation is through its direct path, it has no explicit control over how much of $h_s$ flows to $h_d$ along the runway through intermediate tokens. As discussed in subsequent sections, we argue that a \textit{misalignment} between direct-path attention and runway accumulated information can result in redundant or semantically irrelevant information to cascade to tokens even if the attention pattern is adequately learned.

\subsection{Runway Induced Representation Bottlenecks}\label{sec:runway_bottlenecks}
The effects of the runway in causal graphs as discussed in \Cref{sec:runway_info_flow} can be formalized using the Jacobian of node/token representations. The Jacobian of node representations is a measure of its sensitivity with respect to a given node, and is often used in GNN literature to upper-bound and analyze representation bottlenecks and over-squashing effects.

\begin{theorem}\label{thm:osq_layers}
    Assume a source token $h_s$ at layer $L\ge 0$ and destination token $h_d$ at $r \ge 0$ layers apart and on arbitrary positions $s \le d$. In addition, let $C >0$ be some constant and $A$ be lower-triangular, row-stochastic adjacency operator. Then, the representational sensitivity of $h_d$ with respect to $h_s$ can be bounded as:

    \begin{equation}
        \left\|\frac{\partial h^{(L+r)}_d}{\partial h^{(L)}_s}\right\| \le C^{r}(\prod^{r-1}_{t=0}(I + A^{(L+t)}))_{ds}\nonumber
    \end{equation}
\end{theorem}

The full theorem and its proof are provided in the Appendix (\Cref{sec:osq_for_trans}).

\Cref{thm:osq_layers} states that representation sensitivity of token $h_d$ with respect to $h_s$ can be upper-bounded by the accumulation of attention patterns $A_{ds}$ between $h_d$ and $h_s$ across layers with the residual stream $I$ reinforcing self-representations. Our findings are both position and layer agnostic and describes representation sensitivities of tokens at arbitrary positions at arbitrary layers. 

In addition, \Cref{thm:osq_layers} also suggests that tokens in the middle of the sequence may be most susceptible to being used as information propagators. The longer the sequential distance between source and destination tokens, the longer their runway and the more intermediate tokens (towards  the middle of the sequence) get involved. This may lead to tokens toward the middle of the sequence being \textit{over-used} as runways for various token pairs at either ends of the sequence, perhaps even at the cost of losing their own representation sharpness. This aligns well with the observed lost-in-the-middle effects \cite{liu2024lost} where performance tends to degrade in retrieval tasks (where representation sharpness is particularly important) towards the middle of the sequence and show a U-shaped performance curve \cite{barbero2024transformers}.

In addition, \Cref{thm:osq_layers} is constructed such that the contribution of the direct path can be separated from the accumulating influence of the runway. Namely, the RHS of \Cref{thm:osq_layers} can be expanded as:

\begin{align}
    \left\|\frac{\partial h^{(L+r)}_d}{\partial h^{(L)}_s}\right\| &\le C^{r+1}\left(\left(\prod^{r-2}_{t=0}(I+A^{(L+t)})\right)_{ds}\right. \\
    &\left.+ \sum_{w \in \mathcal{N}(d)}A^{(L+r-1)}_{dw}\left(\prod^{r-2}_{t=0}(I+A^{(L+t)})\right)_{ws}\right)\nonumber
\end{align}

where the first term inside the main paranthesis $\left(\prod^{r-1}_{t=0}(I+A^{(L+t)})\right)_{ds}$ is the self-token update term $\partial \phi_1$ and the second term shows the structure of cross-token mixing, isolated as:

\begin{align}
    \sum_{w \in \mathcal{N}(d)}A^{(L+r-1)}_{dw}\left(\prod^{r-2}_{t=0}(I+A^{(L+t)})\right)_{ws}
\end{align}

Here, $A^{(L+r-1)}_{dw}$ denotes the direct-path attention pattern for $h_d$ (over its immediate neighbors $w \in \mathcal{N}(d)$) and $\left(\prod^{r-2}_{t=0}(I+A^{(L+t)})\right)_{ws}$ denotes the accumulating influence of the runway. This shows that the destination token's own attention pattern (over its row) acts as a gate that chooses a mixture of indirectly propagated information along its runway. However, $A_{dw}$ is only conditioned on the individual neighbor tokens $h_w, \; \forall w \in \mathcal{N}(d)$ without explicit access to the global landscape of the runway, which can lead to $A_{dw}$ being unaware of global cascade of redundant or semantically uninformative information.

\begin{lemma}[Blindspots for standard attention mechanism]\label[lemma]{lem:softmax_blindspot}
Let token representation $h_w = \widetilde{h}_w + \delta$ where $\widetilde{h}_w$ is (unique) token representation that differentiates $h_w$ and $\delta$ denotes the runway influences on $h_w$.  Assume also that $\delta$ can be further decomposed into $\delta = \delta_c + \delta_r$ where $\delta_c$ is the row-wise perturbation (common mode influence) and $\delta_r$ is the residual deviation from row-wise perturbation. Then, attention weights cannot discern runway influences except through (small) residual deviations as:
\begin{align}
    \left\|a_{d} - \widetilde{a}_{d}\right\| \le \sigma_0P\left\|\delta_r\right\|\nonumber
\end{align}
\end{lemma}

Here, $\sigma_0$ represents the global Lipschitz constant and $P$ represents the bound on model terms. Full lemma and its proof are provided in \Cref{sec:runway_cascade}.

\Cref{lem:softmax_blindspot} states that when runway influences largely carry common mode information (i.e. shared information between tokens), the attention weights increasingly become insensitive toward such global influences. Viewing attention mechanism as a contractive mixing operation whose repeated application drives representations closer together \cite{wu2023demystifying,barbero2025why,arroyo2025bridging}, these common mode influences may grow larger with model depth. In addition, these common mode influences are invariant to attention aggregation and cascade forward in attention mechanism.

\begin{theorem}[Runway Cascade]\label{thm:runway_cascade}
 Let token representation $h_w = \widetilde{h}_w + \delta$ where $\widetilde{h}_w$ is (unique) token representation that differentiates $h_w$ and $\delta$ denotes the runway influences on $h_w$.  Assume also that $\delta$ can be further decomposed into $\delta = \delta_c + \delta_r$ where $\delta_c$ is the row-wise perturbation (common mode influence) and $\delta_r$ is the residual deviation from row-wise perturbation. Then, their common mode runway influence must persist and cascade forward in the attention mechanism.
\end{theorem}

The full theorem and its proof are provided in \Cref{sec:runway_cascade}.

The implication of runway cascade is that redundant and irrelevant information can have a snowballing effect where the more they influence the global runway landscape, the more they become invisible to the attention mechanism and cascade forward, creating more influence. This suggests that information propagation in causal transformers may be improved if the attention mechanism is explicitly conditioned on runway influences. With this insight, we introduce a runway-aware rewiring strategy for the causal graph.

\section{Runway Aware Rewiring}\label{sec:runway_aware_rewiring}
With the observation that direct attention patterns in their standard formulation may be inadequate in accounting for runway effects, we introduce a runway-aware rewiring strategy for causal attention. Instead of simply computing attention on individual neighbor tokens where common mode redundancies may carry forward unnoticed, we contextualize each token against a global runway summary and rewire their attention pattern accordingly. Our rewiring is a soft-rewiring approach that scales each edge in the causal graph with computed runway-coefficients.

At the core of our rewiring strategy is quantifying how much of a given token $h_m$ in the runway $\mathcal{R}_{ds}$ already carries runway-accumulated redundancies with respect to the destination token $h_d$ where $s \le m < d-1$. We denote this quantity as \textit{runway-coefficient} $r_{dm}$ which is calculated as follows:

\begin{align}
    r_{dm} = \sigma(\tau(h_{d-1}, h_{m}))
\end{align}

where given $h_d$, the token that immediately precedes it in the sequence $h_{d-1}$ computes a compatibility score against each $h_m$ for all $m, s \le m < d-1$ to evaluate how much of the runway is already dominated by $h_m$ through indirect paths. Intuitively, if $h_{d-1}$ is highly compatible with $h_m$, then this indicates that $h_m$ is sufficiently represented in the runway via indirect paths already and thus its contribution should be scaled down proportionally for $h_d$. In this way, the past-adjacent token $h_{d-1}$ serves as a measure of the runway for $h_d$, giving context for each $h_m$ to measure against. We choose $h_{d-1}$ because it represents the most up-to-date summary of the runway for each $h_d$, making it an appropriate measure of runway landscape (we also experiment using previous $K$ tokens as the runway signal in our ablation studies in \Cref{sec:ablations}). The raw compatibility measure $\tau(\cdot)$ can be calculated with a dot-product (\Cref{eq:runway_dot}) to keep it parameter-free  or with a bilinear form (\Cref{eq:runway_bilin}) to inject some learnability:

\begin{align}
    \tau_{\text{dot}}(h_{d-1}, h_{m}) = h_{d-1}^T\cdot h_{m}\label{eq:runway_dot}\\
    \tau_{\text{bilin}}(h_{d-1}, h_{m}) = h_{d-1}^T B h_{m}, \; B \in \mathbb{R}^{N \times N}\label{eq:runway_bilin}
\end{align}

We use the dot-product formulation (\Cref{eq:runway_dot}) in our main experiments to make sure our rewired model introduces no additional parameters for fair comparison against the standard transformer. We do perform ablations with the bilinear form (\Cref{eq:runway_bilin}) in \Cref{sec:app_ablation_studies} and find that it is comparable in performance to the dot-product approach.

Finally, a smooth thresholding function like sigmoid is applied as $\sigma(\cdot)$ to ensure $0 < r_{dm} < 1$.

As large $r_{dm}$ indicates sufficient representation of $h_m$ in the runway, the direct-path attention between $h_d$ and $h_m$ should then be scaled down proportionally. The multiplicative scaling is done using the complement $\overline{r}_{dm} = 1 - r_{dm}$ to ensure larger runway-coefficient leads to stronger down-scaling. Namely, given the original edge-weights of the causal transformer:

\begin{align}
    e_{dm} = \text{SoftMax}_{\mathcal{N}(d)}(\frac{(h_dW_Q)^T \cdot (h_mW_K)}{\sqrt{d_k}})
\end{align}

the edge weights are scaled (\Cref{eq:scale}) and re-normalized (\cref{eq:renorm}) to ensure each row sums up to 1.0 as: 

\begin{align}
    \widetilde{e}_{dm} = e_{dm} \cdot \overline{r}_{dm}\label{eq:scale}\\
    \widehat{e}_d = \frac{\widetilde{e}_d}{\sum_m \widetilde{e}_{dm}}\label{eq:renorm}
\end{align}

Finally, the complete mechanism for runway-aware rewiring for $h_d$ can be expressed as:

\begin{align}
    h^{(l+1)}_d = \sum_{m} \widehat{e}^{(l)}_{dm}(h^{(l)}_mW_V)
\end{align}

Due to the reliance on the immediate preceding token $h_{d-1}$ for compatibility calculations, they are never scaled down for $h_d$, forming a directed line-graph like topology among tokens wherein information propagates without re-wiring. Likewise, each token's self connection is also preserved without down-scaling. In our setup, edges are always down-scaled and never explicitly boosted by runway coefficients. However, due to re-normalization and the mechanics of softmax, edges with relatively small down-scaling effects end up getting boosted overall relative to edges with more severe down-scaling. This way, our rewiring can not only dampen unnecessarily dominant signals but it can also boost over-squashed signals later on to re-introduce them into the representation landscape if need be.

In implementation, our rewiring strategy re-purposes the value vectors from one of the attention heads from standard causal attention to calculate the runway-coefficients and can be seamlessly added on top of the standard attention block. For more information on implementation details and precise accounting of parameter counts, refer to \Cref{sec:impl_details}.

\section{Experiments and Results}
We train two decoder-only transformers with multi-head attention: (1) standard transformer for baseline comparisons and (2) our rewired transformer for runway-aware modeling. Unless specifically stated otherwise, we use the dot-product formulation from \Cref{eq:runway_dot} for our rewired model which introduces no additional parameters on top of the standard transformer. We use the C4 dataset \cite{raffel2020exploring} and train them for next-token prediction. We train models of various sizes, from 50 million parameters up to 750 million parameters, all trained on 20 billion tokens using a global batch size of 256. Unless stated otherwise, context window of 1024 is used as the sequence length. We use Pre-LN \cite{xiong2020layer} instead of Post-LN and use rotary positional encodings (RoPE) \cite{su2024roformer}. We use the Adam optimizer \cite{kingma2014adam} and all models are trained on NVIDIA A100 GPUs. For modeling and implementation details, refer to \Cref{sec:impl_details}.

\subsection{General Language Modeling}\label{sec:gen_modeling}

\begin{figure*}[t]
  \centering
  \begin{subfigure}[t]{0.48\textwidth}
    \centering
    \includegraphics[width=\linewidth]{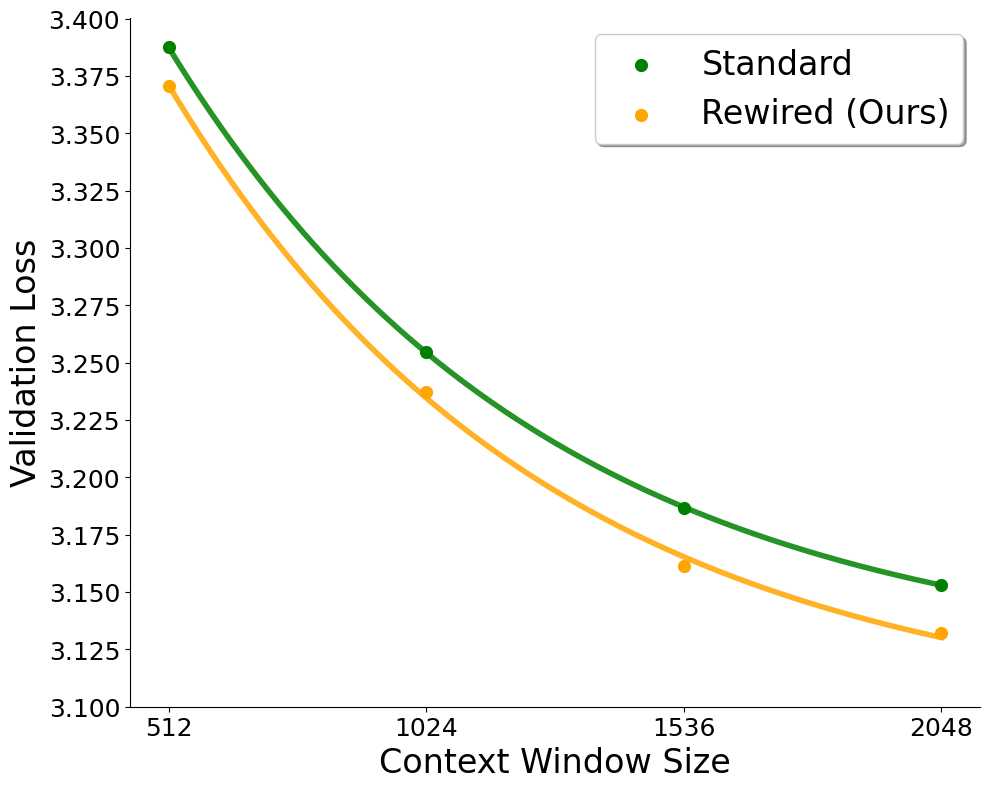}
    \caption{}
    \label{fig:language_seqlen}
  \end{subfigure}\hfill
  \begin{subfigure}[t]{0.48\textwidth}
    \centering
    \includegraphics[width=\linewidth]{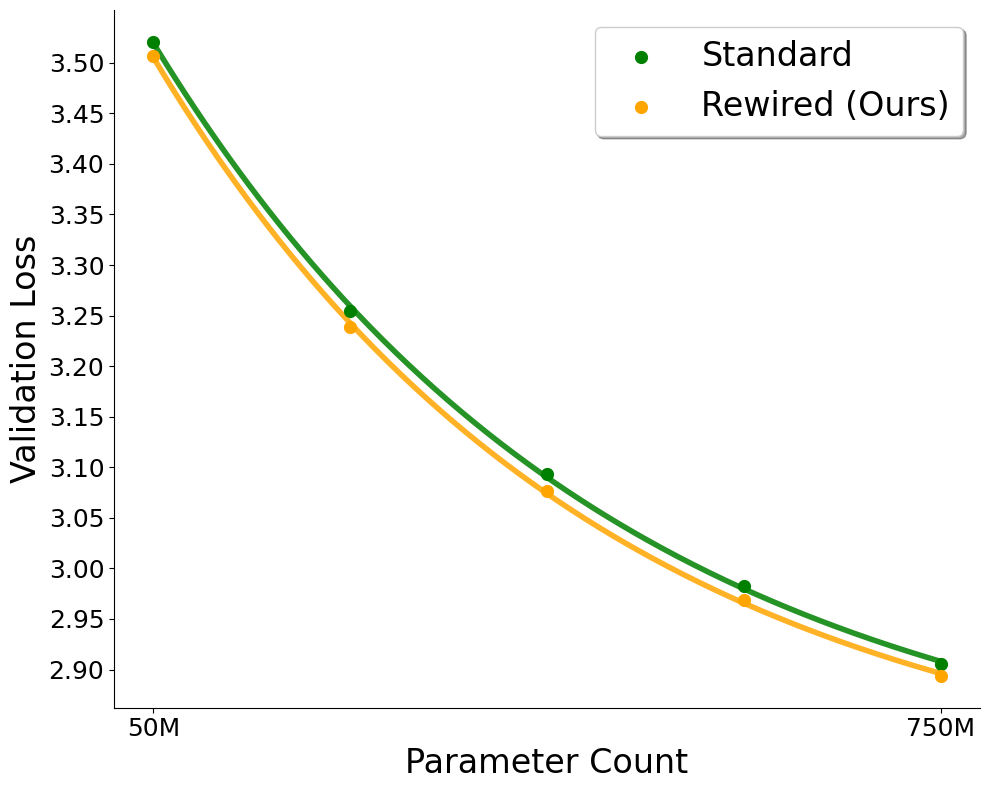}
    \caption{}
    \label{fig:language_ppl}
  \end{subfigure}

  \caption{Validation perplexity on the C4 dataset with (a) varying context window sizes (from 512 to 2048) using a 100 million parameter model and (b) varying model sizes from 50 million to 750 million parameters. \textit{Note: lower values indicate better performance.}}
  \label{fig:language}
\end{figure*}

\paragraph{Validation Performance.}
The main language modeling performance comparisons between standard and our rewired transformers are shown in \Cref{fig:language}. On the left, \Cref{fig:language_seqlen} compares their validation perplexity at varying context window sizes for in-context evaluation. Both standard and rewired transformers are trained with 4 different context window sizes (From 512 to 2048) with 100 million parameters and evaluated at their corresponding train-time context window size. We observe that the rewired transformer outperforms the standard transformer, with its improvements growing with context window size. With growing context window size, so too grows the available runway lengths and thus it is reasonable to expect that our rewired model should show growing improvements compared to standard transformer. On the right, \Cref{fig:language_ppl} compares validation perplexity for both models at various model sizes, from 50 million to 750 million parameters. The context window size is fixed at 1024. Here we observe modest but consistent performance improvements for our rewired transformer likely stemming from our runway-aware rewiring reducing noise and redundancy accumulation and thus keeping token representations slightly sharper. We also compare reasoning performance in \Cref{sec:app_reasoning} and generally observe improved performance from our rewired model. 

\subsection{Information Retrieval Analysis}
We perform passkey retrieval analysis to compare copying and information retrieval abilities of our rewired transformer against the baseline standard transformer. For both standard and rewired transformers, we use 450 million parameter models trained with general language modeling using 1024 as the sequence length from \Cref{sec:gen_modeling}. For this task, a random 5-digit passkey is embedded into the sequence of natural language where the objective is to retrieve the passkey. We only count exact matches and thus all and only the correct 5 digits of the passkey must be retrieved for it to count as successful retrieval. For more information on the experimental setup, refer to \Cref{sec:app_info_retrieval}.

For the main retrieval analysis in \Cref{fig:retrieval_main}, we compare performance of both models at varying context window sizes, starting in-context at 1024 and increasing the context window size up to 1536. Passkeys are embedded at 10 different depths for each context length and each context length's performance is run 10 times. Their overall results are averaged and the light shaded regions represent their standard deviation. 

As shown in \Cref{fig:retrieval_main}, we find that our rewired transformer shows superior retrieval performance compared to the standard transformer. Furthermore, we observe that our rewired version maintains performance for longer while standard transformer's performance collapses much more quickly when taken out of in-context regime. Notably, at context window length of 1280, the standard transformer already shows less than 20\% accuracy while our rewired version maintains reasonable performance at above 60\%. In fact, the rewired transformer's performance at context length 1280 is comparable to the standard transformer's performance at context length 1216. We perform the same experiment using longer context length (up to 2560) on 120 million parameter models trained at context length of 2048 in \Cref{fig:retrieval_main_long_d12}. We observe the same pattern of our rewired transformer not only showing superior performance but holding reasonable performance for much longer while the standard transformer's retrieval performance quickly collapses to 0\% out of context. The superior performance of our rewired transformer beyond the training context length across model scales is evidence that runway awareness may contribute to improved long-range information routing.

Next, we perform depth-based comparisons in \Cref{fig:retrieval_depth}. The same models from \Cref{fig:retrieval_main} now have their accuracies measured at various depths of passkey placement within the sequence with the sequence length fixed at 1216. Performance at each depth is measured 10 times and averaged with the shaded regions showing the standard deviation. 

Not only do we observe superior performance for the rewired model at every depth, but we also see that the U-shaped curve from lost-in-the-middle effect is completely remedied in the rewired version. In contrast, standard transformer clearly exhibits the U-shaped curve in its performance and shows degradation towards the middle of the sequence, matching observations from prior works \cite{liu2024lost,barbero2024transformers}. Overall, superior retrieval results despite identical parameter counts suggest that our rewired transformer is better at maintaining representation sharpness and token-level distinctness as the context grows, rather than letting tokens get indistinguishably mixed along the indirect pathways.

\begin{figure}[t]
  \centering
  \begin{subfigure}[t]{0.32\linewidth}
    \centering
    \includegraphics[width=\linewidth]{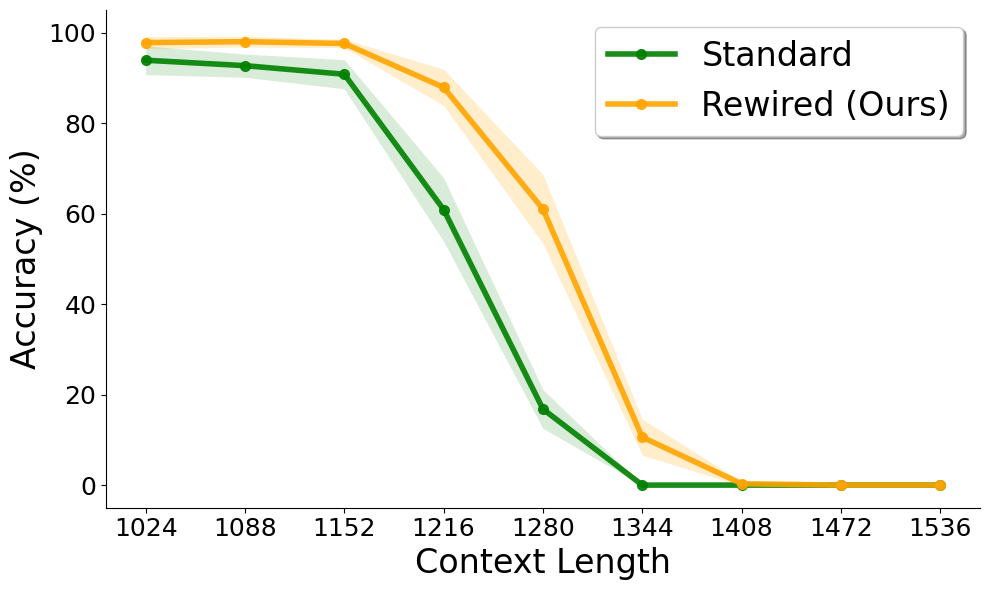}
    \caption{}
    \label{fig:retrieval_main}
  \end{subfigure}
  \hfill
  \begin{subfigure}[t]{0.32\linewidth}
    \centering
    \includegraphics[width=\linewidth]{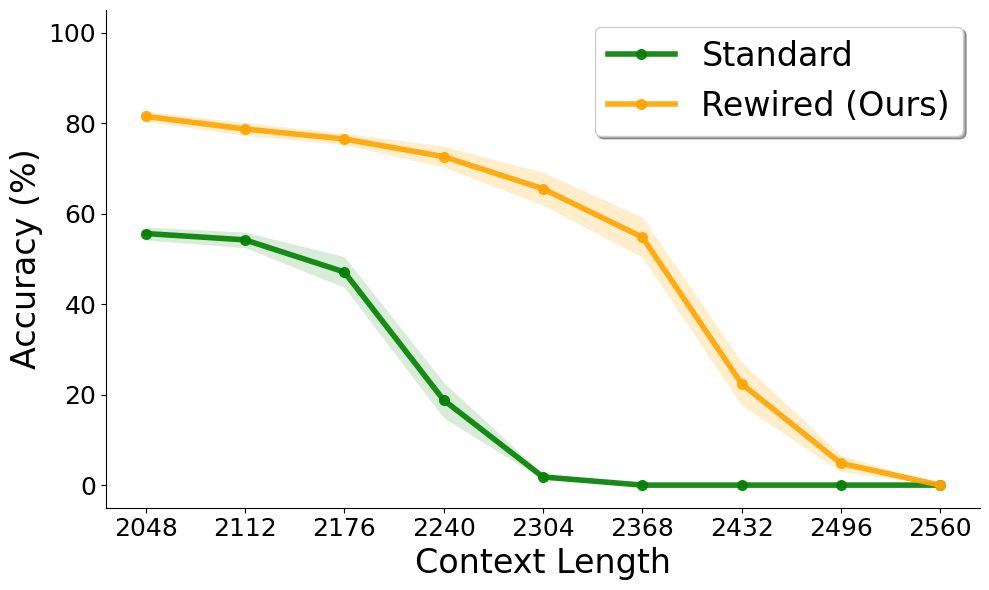}
    \caption{}
    \label{fig:retrieval_main_long_d12}
  \end{subfigure}
  \hfill
  \begin{subfigure}[t]{0.32\linewidth}
    \centering
    \includegraphics[width=\linewidth]{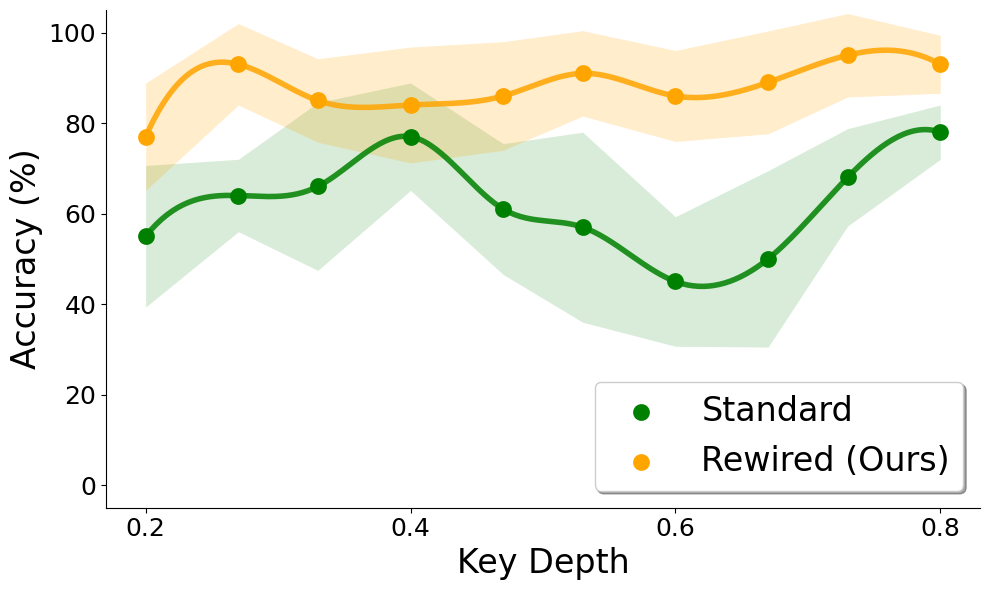}
    \caption{}
    \label{fig:retrieval_depth}
  \end{subfigure}
  \caption{(a) Passkey retrieval test for standard and our rewired transformers. We use 450 million parameter models trained on general language tasks at context length of 1024 and test their retrieval abilities up to context length of 1536. (b) An extension of the same experiment from (a) using longer context (up to 2.5k) on 120 million parameter models (c) Zero-shot passkey retrieval test for standard and our rewired transformers where passkeys are embedded at varying depths. \textit{Note: higher values indicate higher accuracy and therefore better performance.}}
  \label{fig:combined_results}
\end{figure}

\subsection{Additional Analyses}
We perform Dirichlet energy analysis to better understand how our rewiring technique changes the learned information mixing strategy in transformers. Dirichlet energy is a measure of token representation distinctness (sharpness) per layer and is used frequently in GNN literature to validate theoretical claims about representational sharpness. Higher energy means more distinct tokens and lower energy means tokens are more similar to each other. We find that our rewired model maintains consistently higher energy (meaning token representations remain more distinct) through the early-middle layers and converges to stability in later layers (more information can be found in \Cref{sec:app_de}). We also provide information on training/inference throughput and computational/memory overhead of our model in \Cref{sec:app_overhead}.

In addition, we compare models' performance on the PG-19 dataset \cite{rae2019compressivetransformerslongrangesequence} in \Cref{sec:app_pg19} to specifically measure long-range modeling performance. We see that our rewired model shows noticeable gains compared to standard transformer and that the performance gap widens with growing sequence length. For out-of-context long-range modeling, we perform extrapolation analysis in \Cref{sec:app_extrapolation_analysis}. Overall, we observe that standard transformers need around 150 million more parameters than our rewired models to achieve comparable performance across model scales.


\subsection{Ablations}\label{sec:ablations}
We compare the dot-product rewiring approach with the bilinear form in \Cref{sec:app_ablation_studies}. We find that they are comparable in performance across tasks. In addition, we study whether using $h_{d-1}$ as the runway heuristic is the right choice in \Cref{sec:app_global_summary}. Namely, we experiment using an average of all preceding tokens (global summary) as the runway signal and find that they result in comparable performance to our original rewiring technique. Finally, we test whether runway-aware rewiring can simply be replaced by uniform, distance-based decay in \Cref{sec:app_alibi}. To do this, we compare our rewired model against transformers using ALiBi \cite{press2022trainshorttestlong} and find that simple distance-based decay is not enough to match the performance of our rewiring. This is further substantiated by our rewiring pattern analysis in \Cref{sec:att_pattern_analysis} where we find that our rewiring mechanism learns complex, non-trivial rewiring patterns in practice.

\section{Limitations}\label{sec:limitations}
Although runway-aware rewiring adds no parameters, it adds computation. In
our current implementation, training throughput is $31\%$ lower and inference
throughput $10\%$ lower than the standard transformer, with $14\%$ higher peak
inference memory (\Cref{sec:app_overhead}). Whether this cost is worthwhile
depends on the setting. The benefits are largest where our analysis predicts
runway effects to matter most: long contexts and extrapolation beyond the
training length. There, rewired models match or exceed standard models one
size larger (\Cref{sec:app_extrapolation_analysis}). On PG-19, the loss gap grows from $0.03$ in-context to $0.24$ at twice the training length. In passkey
retrieval at a context length of 1280, rewired models retain over $60\%$
accuracy where standard models fall below $20\%$.

In contrast, gains are modest for language modeling at the training context
length (approximately $0.02$ in validation loss) and for zero-shot
reasoning. On reasoning benchmarks, we see improvements of $0.1$--$1.5$
points on four of five benchmarks and a $0.9$-point decrease on ARC-Easy. In
these regimes, the additional compute may not be justified, particularly for
compute-constrained training.

We view the current mechanism as a first instantiation rather than the most
efficient one. The phenomenon we identify, runway cascade, does not depend on
this particular design, and cheaper ways of conditioning attention on runway
context are a natural direction for future work. These include fused kernels
(\Cref{sec:fused_attention}), computing coefficients in only a subset of
layers, or applying rewiring only beyond a distance threshold.

\section{Conclusion}
In this paper, we introduced runway cascade in transformers as the misalignment between direct-path attention and compounding influence of indirect paths along the causal graph. We performed theoretical analysis to examine information propagation patterns in decoder-only transformers and showed how certain redundancies and irrelevant information can cascade forward, invariant to attention mechanism. Building on these theoretical observations, we introduced runway-aware rewiring to better align the attention mechanism with runway influences. Empirically, our rewiring approach showed superior modeling capabilities compared to standard transformers across language tasks.

\bibliography{lee_neurips26}


\appendix

\section*{Broader Impact Statement}
This paper presents work whose goal is to advance the field of Machine
Learning, specifically in language modeling. There are many potential societal consequences of our work, none which we feel must be specifically highlighted here.

\section{Proofs}

\subsection{Token Sensitivity Analysis}\label{sec:osq_for_trans}

In order to draw sensitivity bounds on causal transformers in terms of the relative runway of the given token pair, it needs to account for arbitrary tokens and arbitrary layers. However, most over-squashing bounds in GNN literature only bound the Jacobian of a node at the final layer with respect to a node at the input layer. In addition, most works also assume a fixed, binary graph structure, which can simply be expressed as a power of the number of layers in the Jacobian bound. However, since transformers employ dynamic attention patterns in each layer, the fixed topology bound does not apply directly. As such, we draw a general over-squashing bound using the Jacobian of arbitrary nodes $h_d, h_s$ at arbitrary layers $L+r$ and $L$, for $L \ge 0$ and $r \ge 0$. This is formulated below with \Cref{thm:app_osq_layers} which incorporates over-squashing bounds for any two nodes at any two layers and accounting for dynamic attention patterns in each layer. In our proofs, we make the simplifying assumption that attention weights are independent of $h$, following the convention of \cite{barbero2024transformers}. Likewise, we also assume supplementary operations such as layer-norms and residual connections are absorbed into $\phi$ and bounded. We use these simplifying assumptions to draw clean bounds highlighting information propagation patterns for interpretive purposes. In reality, the attention weights would be a function of token representations but their corresponding information propagation patterns from a topological point-of-view would remain unchanged from our analysis.

\begin{theorem}\label{thm:app_osq_layers}
    Assume a token $h_s$ at layer $L\ge 0$ and another token $h_d$ at $r \ge 0$ layers apart and on arbitrary positions $s \le d$. In addition, let $C >0$ be some constant and $A$ be lower-triangular, row-stochastic adjacency operator. Then, the representational sensitivity of $h_d$ with respect to $h_s$ can be bounded as:

    \begin{equation}
        \left\|\frac{\partial h^{(L+r)}_d}{\partial h^{(L)}_s}\right\| \le C^{r}(\prod^{r-1}_{t=0}(I + A^{(L+t)}))_{ds}\nonumber
    \end{equation}
\end{theorem}

\begin{proof}
    We prove this by induction on layer $r$ following the style of \cite{black2023understanding}. For the base case of $r=0$, either $d = s$ or $d \neq s$, and thus

    \begin{equation}
    \frac{\partial h^{(L)}_d}{\partial h^{(L)}_s} =
    \begin{cases}
        I & \text{if } d=s,\\
        0 & \text{if } d\neq s,
    \end{cases}
    \nonumber
    \end{equation}
    
    where $I$ is the (feature-dimension) identity matrix and $0$ is the zero matrix. Therefore,
    \begin{equation}
        \left\|\frac{\partial h^{(L)}_d}{\partial h^{(L)}_s}\right\| =
        \begin{cases}
            I & \text{if } d=s,\\
            0 & \text{if } d\neq s,
        \end{cases}
        = \delta_{ds} \le (I)_{ds}.
        \nonumber
    \end{equation}

    Assume that the statement holds for some $r \ge 0$. We now prove the case for layer $r+1$ by induction.
    \begin{align}
        \left\|\frac{\partial h^{(L+r+1)}_d}{\partial h^{(L)}_s}\right\| &= \left\|\partial_1\phi_{L+r}\cdot \frac{\partial h^{(L+r)}_d}{\partial h^{(L)}_s} + \partial_2\phi_{L+r}\sum_{w \in \mathcal{N}_d}A^{(L+r)}_{dw}\partial\psi_{L+r}\cdot \frac{\partial h^{(L+r)}_w}{\partial h^{(L)}_s}\right\|\nonumber\\
        &\le \left\|\partial_1\phi_{L+r}\right\|\cdot \left\|\frac{\partial h^{(L+r)}_d}{\partial h^{(L)}_s}\right\| + \left\|\partial_2\phi_{L+r}\right\| \cdot \left\|\partial\psi_{L+r}\right\|\cdot \sum_{w \in \mathcal{N}_d}A^{(L+r)}_{dw}\cdot \left\|\frac{\partial h^{(L+r)}_w}{\partial h^{(L)}_s}\right\| \nonumber\\
    \intertext{Let $C \ge \left\|\nabla \phi\right\|$ and $C \ge     \left\|\nabla \phi\right\| \cdot \left\|\nabla \psi\right\|$,}\nonumber\\
        &\le C\cdot \left\|\frac{\partial h^{(L+r)}_d}{\partial h^{(L)}_s}\right\| + C\cdot \sum_{w \in \mathcal{N}_d}A^{(L+r)}_{dw}\cdot \left\|\frac{\partial h^{(L+r)}_w}{\partial h^{(L)}_s}\right\| \nonumber\\
        &\le C C^r(\prod^{r-1}_{t=0}(I+A^{(L+t)}))_{ds} + C \sum_{w \in \mathcal{N}(d)}A^{(L+r)}_{dw}C^r(\prod^{r-1}_{t=0}(I+A^{(L+t)}))_{ws}\tag{induction}\nonumber\\
        &= C^{r+1}\left(\left(\prod^{r-1}_{t=0}(I+A^{(L+t)})\right)_{ds} + \sum_{w \in \mathcal{N}(d)}A^{(L+r)}_{dw}\left(\prod^{r-1}_{t=0}(I+A^{(L+t)})\right)_{ws}\right)\nonumber\\
    \intertext{Defining $A^{(l)}_{dw} = 0$ for $w \notin \mathcal{N}(d)$}\nonumber\\
        &= C^{r+1}\left((I + A^{(L+r)})\prod^{r-1}_{t=0}(I+A^{(L+t)})\right)_{ds}\tag{definition of matrix multiplication}\nonumber\\
        &= C^{r+1}\left(\prod^{r}_{t=0}(I+A^{(L+t)})\right)_{ds}\nonumber\\
    \intertext{Therefore,}\nonumber\\
        \left\|\frac{\partial h^{(L+r+1)}_d}{\partial h^{(L)}_s}\right\| &\le C^{r+1}\left(\prod^{r}_{t=0}(I+A^{(L+t)})\right)_{ds}\nonumber
    \end{align}

    This concludes the proof for \Cref{thm:app_osq_layers}.
\end{proof}

\subsection{Blindspots for Softmax Attention and Runway Cascade}\label{sec:runway_cascade}
In order to construct our theoretical results on blindspots for attention mechanism, we first quickly review the translation invariance property of the softmax function.

\begin{lemma}[Translation invariance of softmax]\label[lemma]{lem:trans_inv}
    $\text{SoftMax}(z_i) = \text{SoftMax}_j(z_i + F)$ where $F$ denotes some common factor for all $z_i, \; i \in \mathcal{N}(j)$
\end{lemma}
\begin{proof}
    \begin{align}
        \text{SoftMax}_j(z_i + F) &= \frac{e^{(z_i + F)}}{\sum_j e^{(z_j + F)}}\nonumber\\
        &= \frac{e^F e^{z_i}}{\sum_j e^F e^{z_j}} = \frac{e^F e^{z_i}}{e^F\sum_j e^{z_j}}\nonumber\\
        &= \frac{e^{z_i}}{\sum_j e^{z_j}}\nonumber\\
    \intertext{Therefore, $\text{SoftMax}_j(z_i + F) = \text{SoftMax}_j(z_i)$.}\nonumber
    \end{align}
    This concludes the proof for \Cref{lem:trans_inv}.
\end{proof}

\begin{lemma}[Blindspots for standard attention mechanism]\label[lemma]{lem:app_softmax_blindspot}
Let token representation $h_w = \widetilde{h}_w + \delta$ where $\widetilde{h}_w$ is (unique) token representation that differentiates $h_w$ and $\delta$ denotes the runway influences on $h_w$.  Assume also that $\delta$ can be further decomposed into $\delta = \delta_c + \delta_r$ where $\delta_c$ is the row-wise perturbation (common mode influence) and $\delta_r$ is the residual deviation from row-wise perturbation. Then, attention weights cannot discern runway influences except through (small) residual deviations as:
\begin{align}
    \left\|a_{d} - \widetilde{a}_{d}\right\| \le \sigma_0P\left\|\delta_r\right\|\nonumber
\end{align}
    
\end{lemma}

\begin{proof}
\begin{align}
    \intertext{We use $\sigma$ to denote the softmax function here.}\nonumber\\
    a_{dw} &= \sigma_{\mathcal{N}(d)}\left(\frac{h_dW_Q(h_wW_K)^T}{\sqrt{d}}\right) = \sigma_{\mathcal{N}(d)}\left(\frac{h_dW_Q((\widetilde{h}_w + \delta_c + \delta_r)W_K)^T}{\sqrt{d}}\right),\; \forall w \in \mathcal{N}(d)\nonumber\\
    \intertext{Expanding the numerator $n = h_dW_Q((\widetilde{h}_w + \delta_c + \delta_r)W_K)^T$,}\nonumber\\
    n &= h_dW_Q(\widetilde{h}_wW_K + \delta_cW_K + \delta_rW_K)^T\nonumber\\
    &= h_dW_Q(\delta_cW_K)^T + h_dW_Q(\widetilde{h}_wW_K + \delta_rW_K)^T\tag{factoring}\nonumber
    \intertext{plugging $n$ back into softmax logit numerator,}\nonumber\\
    a_{dw} &= \sigma_{\mathcal{N}(d)}\left(\frac{h_dW_Q(\delta_cW_K)^T + h_dW_Q(\widetilde{h}_wW_K + \delta_rW_K)^T}{\sqrt{d}}\right)\nonumber\\
    \intertext{Taking $F_{\text{common}} = \frac{h_dW_Q(\delta_cW_K)^T}{\sqrt{d}}$ for all keys $h_wW_K, \; \forall w \in \mathcal{N}(d)$,}\nonumber\\
    a_{dw} &= \sigma_{\mathcal{N}(d)}\left(F_{\text{common}} + \frac{h_dW_Q(\widetilde{h}_wW_K + \delta_rW_K)^T}{\sqrt{d}}\right) = \sigma_{\mathcal{N}(d)}\left(\frac{h_dW_Q(\widetilde{h}_wW_K + \delta_rW_K)^T}{\sqrt{d}}\right)\tag{translation invariance \Cref{lem:trans_inv}}\nonumber\\
    \intertext{Then the Lipschitz bound between $a_d$ and $\widetilde{a}_d$ can be formulated as:}\nonumber\\
    \left\|a_d - \widetilde{a}_d\right\| &\le \sigma_0P\left\|\delta_r\right\|
    \intertext{where $\sigma_0$ is the global Lipschitz constant and $P$ is the bound on model terms for $h_d$'s direct-path attention defined as: $P = \frac{\left\|h_dW_Q\right\|\left\|W_K\right\|}{\sqrt{d}}$}\nonumber\\
    \intertext{Therefore,}\nonumber\\
    \left\|a_{d} - \widetilde{a}_{d}\right\| &\le \sigma_0P\left\|\delta_r\right\|\nonumber
\end{align}
    This concludes the proof for \Cref{lem:app_softmax_blindspot}.
\end{proof}

\begin{theorem}[Runway Cascade]\label{thm:app_runway_cascade}
 Let token representation $h_w = \widetilde{h}_w + \delta$ where $\widetilde{h}_w$ is (unique) token representation that differentiates $h_w$ and $\delta$ denotes the runway influences on $h_w$.  Assume also that $\delta$ can be further decomposed into $\delta = \delta_c + \delta_r$ where $\delta_c$ is the row-wise perturbation (common mode influence) and $\delta_r$ is the residual deviation from row-wise perturbation. Then, their common mode runway influence must persist and cascade forward in the attention mechanism.
\end{theorem}
\begin{proof}
\begin{align}
     \intertext{Let the attention-aggregated message be $m_d = \sum_{w\in\mathcal{N}(d)}a_{dw} h_wW_V$ where $a_{dw} = \text{SoftMax}_{\mathcal{N}(d)}\left(\frac{h_dW_Q\cdot (h_wW_K)^T}{\sqrt{d}}\right)$. We have established in \Cref{lem:app_softmax_blindspot} that each $a_{dw}$ is blind to common mode influences.}\nonumber\\
     m_d &= \sum_{w\in\mathcal{N}(d)}a_{dw}h_wW_V = \sum_{w\in\mathcal{N}(d)}a_{dw}(\widetilde{h}_w + \delta_c + \delta_r)W_V\nonumber\\
     &= \sum_{w\in\mathcal{N}(d)}a_{dw}\widetilde{h}_wW_V + a_{dw}\delta_cW_V + a_{dw}\delta_rW_V\nonumber\\
     &= \sum_{w\in\mathcal{N}(d)}a_{dw}\delta_cW_V + \sum_{w\in\mathcal{N}(d)}\left(a_{dw}\widetilde{h}_wW_V + a_{dw}\delta_rW_V\right)\nonumber\\
     \intertext{Leveraging row-stochasticity of causal attention weights where $\sum_{w\in\mathcal{N}(d)} a_{dw} = 1.0$,}\nonumber\\
     &= \delta_cW_V + \sum_{w\in\mathcal{N}(d)}a_{dw}(\widetilde{h}_w + \delta_r)W_V\nonumber\\
     \intertext{Therefore,}\nonumber\\
     m_d &= \delta_cW_V + \sum_{w\in\mathcal{N}(d)}a_{dw}(\widetilde{h}_w + \delta_r)W_V\nonumber
\end{align}
    This concludes the proof for \Cref{thm:app_runway_cascade}.
\end{proof}

Overall, \Cref{lem:app_softmax_blindspot} and \Cref{thm:app_runway_cascade} illustrate how common mode runway influences are invariant to the attention mechanism and that these influences can cascade forward, potentially resulting in accumulations of noise and redundancies. The accumulation of common mode influences can be explained by theory of over-smoothing in attention-based aggregation schemes, where attention acts as a contractive mixing operator that drives token representations closer together with repeated application \cite{wu2023demystifying,barbero2025why,arroyo2025bridging}. While we isolate the attentional component of the transformer to drive our theoretical conclusions, it should be noted that with the application of additional components such as MLPs and layer-norms, things can become less straight-forward. Specifically, layer-norms have been shown to be able to prevent over-smoothing to the point of complete rank collapse \cite{wu2024role} with the same work also providing counter-examples of when layer-norms cannot prevent over-smoothing. 

\section{Interpreting Causal Transformers as GNNs}\label{sec:app_gnn}
A causal attention mechanism can be interpreted as a message-passing GNN of the form: 

\begin{align}
    h^{(l+1)}_i = \phi\left(h^{(l)}_i, \sum_{j \le i} A_{ij}\psi(h^{(l)}_i, h^{(l)}_j)\right)
\end{align}

where the aggregate function $\psi$ represents the attention-based weighted sum of tokens dictated by lower-triangular and row-stochastic adjacency operator $A$ as: 

\begin{align}
    \psi(h^{(l)}_i, h^{(l)}_j) = h^{(l)}_jW_V\\
    A_{ij} = \text{SoftMax}(h^{(l)}_iW_Q \cdot h^{(l)}_jW_K)
\end{align}

and the update function $\phi$ represents intra-token evolution to the next layer as follows:

\begin{align}
    \phi(h^{(l)}_i, \bar{h}^{(l)}_i) = \text{MLP}(h^{(l)}_i + \bar{h}^{(l)}_i)\\
    \bar{h}^{(l)}_i = \sum_{j \le i} A_{ij}\psi(h^{(l)}_i, h^{(l)}_j)
\end{align}

\section{Implementation Details}\label{sec:impl_details}
In this section, we discuss implementation and comparison details for our rewired transformers and baseline standard transformers.

\paragraph{Training Setup.}
For training, we use the C4 dataset \cite{raffel2020exploring} and train them for standard next-token prediction using 20 billion tokens and global batch size of 256. Unless otherwise stated, context window of 1024 is used as the sequence length. We use Pre-LN \cite{xiong2020layer} instead of Post-LN and use rotary positional encodings (RoPE) \cite{su2024roformer}. We use the Adam optimizer \cite{kingma2014adam}.

\paragraph{Model Parametarization.}
We train models of various sizes, from 50 million parameters up to 750 million parameters. We follow conventions from \cite{esser2024scaling,leviathan2025selective} and set model size with parameter $d$ s.t. $d_{\text{model}} = 64d$ and both number of attention heads and layers correspond to $d$. The precise parameter counts are shown in \Cref{tab:param_count}.

\begin{table}[h]
  \centering
  \caption{Model Parameterization.}
  \label{tab:param_count}
  \begin{tabular}{ccc}
    \toprule
    $d = n_{\text{heads}} = n_{\text{layers}}$ & $d_{\text{model}}$ & Model Parameter Count\\
    \midrule
    8 & 512 & 50,927,104\\
    12 & 768 & 123,607,296\\
    16 & 1024 & 252,890,112\\
    20 & 1280 & 457,726,720\\
    24 & 1536 & 756,969,984\\
    \bottomrule
  \end{tabular}
\end{table}

\paragraph{Runway Aware Rewiring Implementation Details.} 
Here we provide implementation details for our runway-aware rewiring methodology. We re-purpose one of the attention heads to calculate runway-coefficients. The attention head is used for regular MHA as usual. We calculate runway-coefficients once from the re-purposed attention head and use them for all attention heads, instead of calculating runway-coefficients independently for each head. While each attention head typically learns different patterns, we find that this setup results in good performance while achieving memory and computational efficiency. We show pseudo-code for implementing attention with runway-aware rewiring in \Cref{listing:rewiring}.

\section{Compatibility with Fused Attention Kernels}
\label{sec:fused_attention}

All experiments in this paper, for both our rewired model and the standard
transformer baselines, use an unfused attention implementation that
materializes the full $N \times N$ attention matrix. We do not use
FlashAttention \citep{dao2022flashattention} or related fused kernels \citep{dong2025flexattention}, and the
throughput and memory figures \Cref{tab:overhead} reflect this unfused setting. Here we
briefly discuss, at a high level, why runway-aware rewiring is structurally
compatible with such kernels, as a direction for future work.

\paragraph{Rewiring as a pre-softmax logit bias.}
As written in \Cref{eq:scale,eq:renorm}, rewiring scales the post-softmax attention
weights and re-normalizes each row. This may suggest that it requires the
completed attention row. However, the two steps together are equivalent to an
additive bias on the pre-softmax logits. Let $x_{dm}$ denote the attention
logits, so that $e_{dm} = \exp(x_{dm}) / \sum_{k} \exp(x_{dk})$. Since the
softmax normalizer is shared across the row, it cancels under re-normalization:
\begin{equation}
\hat{e}_{dm}
= \frac{e_{dm}\,\bar{r}_{dm}}{\sum_{k} e_{dk}\,\bar{r}_{dk}}
= \frac{\exp\!\left(x_{dm} + \log \bar{r}_{dm}\right)}
       {\sum_{k} \exp\!\left(x_{dk} + \log \bar{r}_{dk}\right)}
= \mathrm{SoftMax}_{\mathcal{N}(d)}\!\left(x_{d} + \log \bar{r}_{d}\right)_{m}.
\end{equation}
Because $\bar{r}_{dm} = 1 - \sigma(\tau_{dm}) = \sigma(-\tau_{dm}) \in (0, 1)$,
the bias is always finite. For edges exempt from
rewiring (such as the self and immediately preceding tokens) the
bias is zero.

\paragraph{Tile-wise computation.}
Crucially, the runway-coefficients depend only on token representations (in
our implementation, the value vectors of the re-purposed head at positions
$d-1$ and $m$) and not on the attention weights. The compatibility scores
$\tau_{dm}$ therefore have the same computational form as $QK^{\top}$: an
inner product between a row operand (value vectors shifted by one position)
and a column operand. A FlashAttention-style kernel that streams key/value
tiles for each block of query rows could compute the corresponding tile of
$\tau$ alongside the $QK^{\top}$ tile, add the bias under the
positional exemptions above, and proceed with the standard online softmax
unchanged. It would never form an $N \times N$ matrix. In the backward pass,
the gradient with respect to the bias equals the gradient with respect to the
logits.

\paragraph{Caveats.}
A fused implementation would require one additional matrix product per tile,
of the same size as $QK^{\top}$. This
cost could be incurred per head, since the coefficients are shared across
heads but fused kernels typically process heads independently. Programmable
interfaces such as FlexAttention \cite{dong2025flexattention} can express the
required masking and a precomputed bias directly, but a fully online variant
would need this extra product computed inside the kernel. We leave
implementing and benchmarking such a kernel to future work.

\section{Experimental Details and Additional Results}\label{sec:additional_experiments}
Here, we provide more details regarding experimental setup as well as some additional results and ablation studies.

\subsection{Reasoning Benchmarks}\label{sec:app_reasoning}
We use 5 reasoning benchmarks to compare performance between standard transformer our rewired transformer. We use the 750 million parameter models trained on standard next-token-prediction task on the C4 dataset and perform zero-shot reasoning on all 6 datasets without finetuning. More information on each benchmark dataset is compiled below.

\paragraph{ARC-Easy/Challenge \cite{clark2018think}.}
AI2 Reasoning Challenge (ARC) are multiple choice science questions, split into easy and challenging variations. ARC-Easy includes questions answerable by simple information retrieval or information co-occurrence methods. On the other hand, ARC-Challenge consists of questions that cannot simply be answered by either information retrieval or information co-occurrence, stressing multi-hop reasoning capabilities. Inputs are short natural language questions with four answer options, and the evaluation metric is normalized accuracy.

\paragraph{HellaSwag \cite{zellers2019hellaswag}.}
HellaSwag is a large-scale common sense benchmark testing the models' abilities to choose plausible continuation of multi-sentence contexts. This task stress-tests procedural knowledge (i.e. "what comes next") in lieu of factual recall. This aspect of HellaSwag distinguishes it from other reasoning benchmarks like the ARC variants and offers evaluation diversity. 

\paragraph{PIQA \cite{bisk2020piqa}.}
Physical Interaction: Question Answering (PIQA) focuses on physical common sense. Each instance presents a shrot goal or description of a physical task and two candidate solutions, and the model must select the physically more plausible option. These scenarios target physical actions like intuitive physics and object usage, so success depends on grounded world knowledge rather than purely linguistic pattern recognition. 

\paragraph{CommonSenseQA \cite{talmor2019commonsenseqa}.}
CommonSenseQA comprises of five-choice questions generated by sampling ConceptNet \cite{speer2017conceptnet} and asking annotators to write questions whose correct answer is the target concept while distractors are other semantically related concepts. The resulting questions emphasize everyday common sense relations such as causes, purposes, locations, and attributes, testing model ability to use structured common sense rather than shallow linguistic pattern recognition.

\paragraph{Performance}
We compare our rewired model against standard transformer on reasoning benchmarks using 750 million parameter models in \Cref{tab:reasoning-results}. For more information on the experimental setup and benchmark datasets used, refer to \Cref{sec:app_reasoning}. Overall, we generally observe modest but consistent improvements across benchmarks for our rewired model (best result for each dataset marked in blue). For ARC variants specifically, ARC-Easy focuses on information co-occurrence which may explain our rewired model's inferior performance as it may over-suppress perceived redundancies in the text. On the other hand, ARC-challenge focuses on multi-hop reasoning and we can see that our rewired model shows a clear advantage compared to standard transformer in this setting.

\begin{table}[h]
    \centering
    \caption{Reasoning performance of standard and our rewired transformers. Best result from each dataset is marked in blue.}
    \label{tab:reasoning-results}
    \resizebox{0.7\columnwidth}{!}{%
    \begin{tabular}{l c c}
        \toprule
        Dataset & \textcolor{ForestGreen}{\textbf{Std. Transformer}} & \textcolor{Orange}{\textbf{Rewired (Ours)}} \\
        \midrule
        ARC-Easy        & \textcolor{blue}{\textbf{43.81\%}} & 42.89\% \\
        ARC-Challenge   & 24.66\% & \textcolor{blue}{\textbf{26.11\%}} \\
        HellaSwag       & 41.75\% & \textcolor{blue}{\textbf{42.49\%}} \\
        PIQA            & 68.61\% & \textcolor{blue}{\textbf{68.82\%}} \\
        CommonsSenseQA   & 19.49\% & \textcolor{blue}{\textbf{19.57\%}} \\
        \bottomrule
    \end{tabular}%
    }
\end{table}



\subsection{Information Retrieval Analysis.}\label{sec:app_info_retrieval}
We perform zero-shot passkey retrieval on models trained on next-token prediction with sequence length of 1024. Starting in-context at sequence length of 1024, we increase the sequence lengths out-of-context to observe performance degradation and failure points of both standard and our rewired transformers. The needle and retrieval prompt are as follows:

Needle: "The passkey is [passkey \#]."

Retrieval Prompt: "The passkey is"

where [passkey \#] is a random 5-digit passkey to be retrieved. The objective is to correctly output the embedded passkey in the next-token prediction following the retrieval prompt. As mentioned, we count exact matches only, so all 5 digits and only the correct digits in the right order must be retrieved. This puts emphasis on representation sharpness and fidelity by design. The needle is placed at designated sequence depth (where lower depth means closer to the beginning of sequence) and the rest of the sequence is populated with natural language from the C4 dataset.

\subsection{Extrapolation Analysis.}\label{sec:app_extrapolation_analysis}
We perform extrapolation on models trained on next-token prediction with sequence length of 1024, same as the retrieval analysis. Starting in-context at 1024, we increase the sequence lengths up to 2048 and observe standard evaluation loss and extrapolation abilities of both standard and our rewired transformers. 

In \Cref{fig:extrap_delta}, we show extrapolation abilities of both standard and rewired transformers at sequence length of 2048 across various model sizes. Both models were trained with in-context sequence length of 1024. We observe consistent improvements for our rewired transformer. In addition, we find that standard transformers need around 150 million more parameters to reach comparable extrapolation performance as our rewired model (refer to \Cref{sec:app_extrapolation_analysis}). This parameter efficiency gains suggest that runway-aware rewiring addresses a fundamental architectural bottleneck of standard attention. As extrapolated contexts create longer runways with more intermediate tokens between source and destination pairs, the progressive accumulation of runway influences may grow more severe, requiring standard transformers to compensate through increased model capacity. In contrast, our rewiring mechanism maintains awareness of runway accumulations by design. The consistent improvements across model size further indicate that our rewiring method is tackling a fundamental architectural limitation of standard transformers, rather than artifacts of insufficient model capacity.

\begin{figure}[t]
  \centering
  \includegraphics[width=0.6\linewidth]{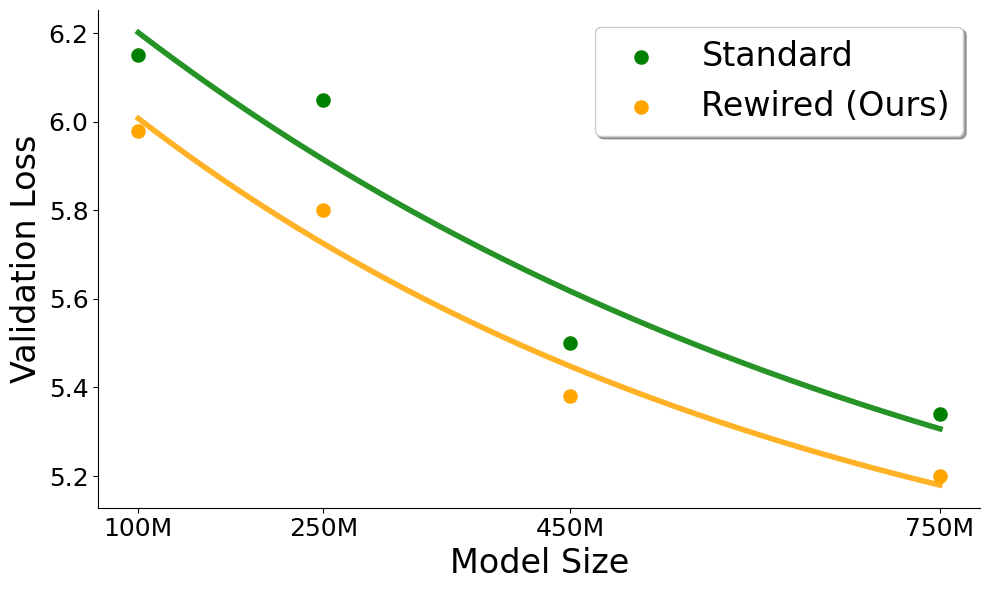}
  \caption{Extrapolation at sequence length 2048 across model sizes for standard and rewired transformers.}
  \label{fig:extrap_delta}
\end{figure}

In \Cref{fig:extrap_multi_small}, we compare the extrapolation abilities of our rewired transformer at 100 million parameters and 250 million parameters against the standard transformer at 250 million parameters. Notably, our rewired transformer at 100 million parameters show better extrapolation performance compared to the standard transformer at 250 million parameters. We observe similar results in \Cref{fig:extrap_multi} comparing standard transformer at 750 million parameters against our rewired transformer at 450 million. Overall, \Cref{fig:extrap_total} shows that standard transformers need around 150 million additional parameters to match our rewired transformers' extrapolation performance and that this persists at different model scales. Intuitively, our runway-aware rewiring can help extrapolation because longer-than-trained contexts create longer runways, so a method that rewires and mitigates progressive runway influences should degrade more gracefully as context length increases.

\begin{figure*}[t]
  \centering
  \begin{subfigure}[t]{0.48\textwidth}
    \centering
    \includegraphics[width=\linewidth]{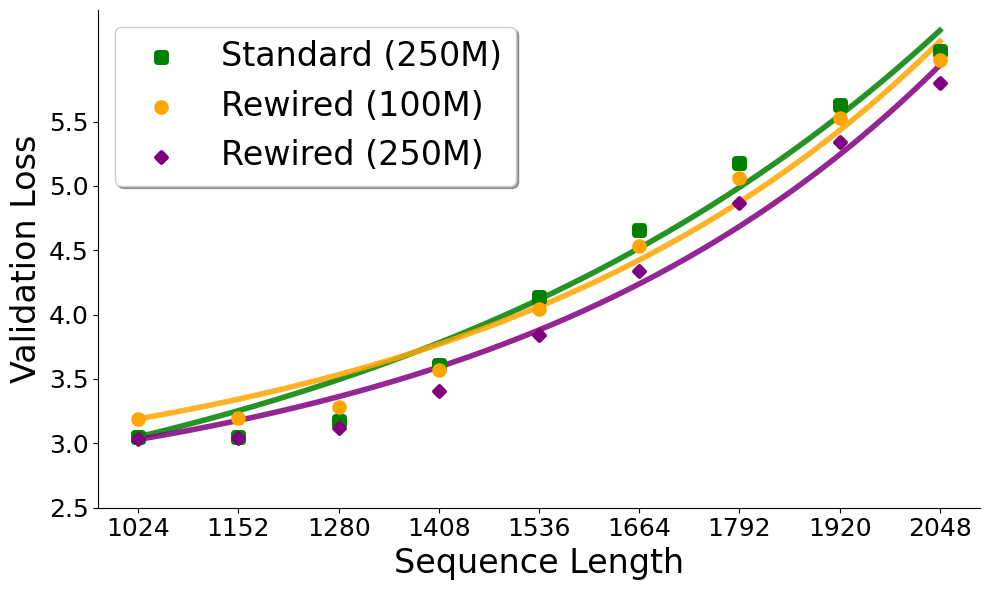}
    \caption{}
    \label{fig:extrap_multi_small}
  \end{subfigure}
  \begin{subfigure}[t]{0.48\textwidth}
    \centering
    \includegraphics[width=\linewidth]{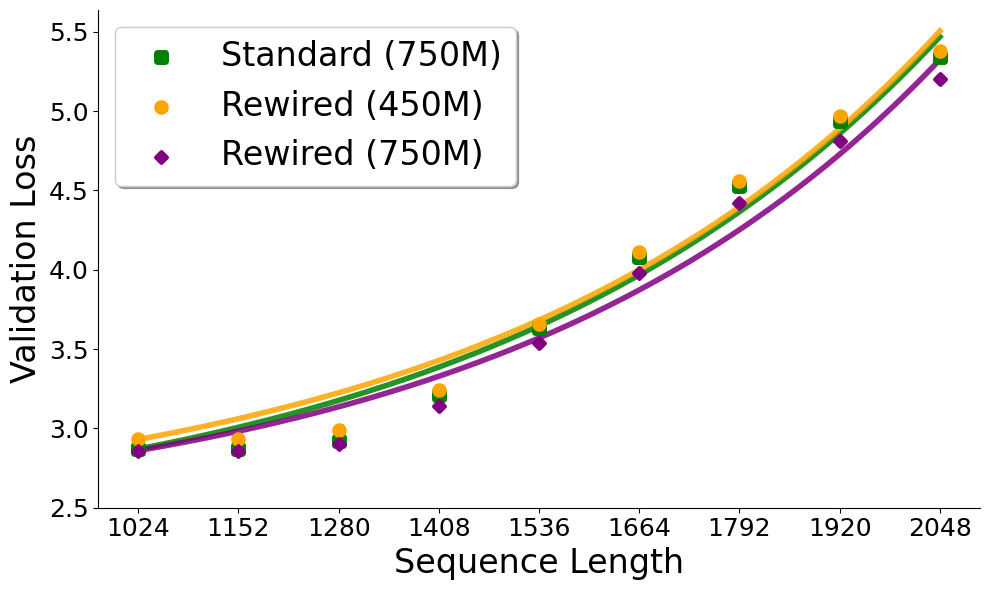}
  \caption{}
  \label{fig:extrap_multi}
  \end{subfigure}\hfill

  \caption{Analysis of extrapolation performance (validation perplexity). Notably, we find that our rewired models show comparable extrapolation performance to standard transformers with 150 million additional parameters.}
  \label{fig:extrap_total}
\end{figure*}

Next, we show additional results in \Cref{fig:extrap_all4}, showing extrapolation comparisons across 4 model sizes (100, 250, 450, and 750 million parameters). As illustrated in \Cref{fig:extrap_all4}, we see that our rewired model generally outperforms standard transformer in extrapolation across model sizes, and that the improvements seem to grow with growing model size. Coupled with findings from \Cref{sec:extrapolation_analysis}, this indicates that our rewired transformer learns effective information routing patterns that translate to longer-than-trained context lengths better than standard transformers, and that this behavior scales favorably with model size. 

\begin{figure}[h]
  \centering

  \begin{subfigure}{0.48\textwidth}
    \centering
    \includegraphics[width=\linewidth]{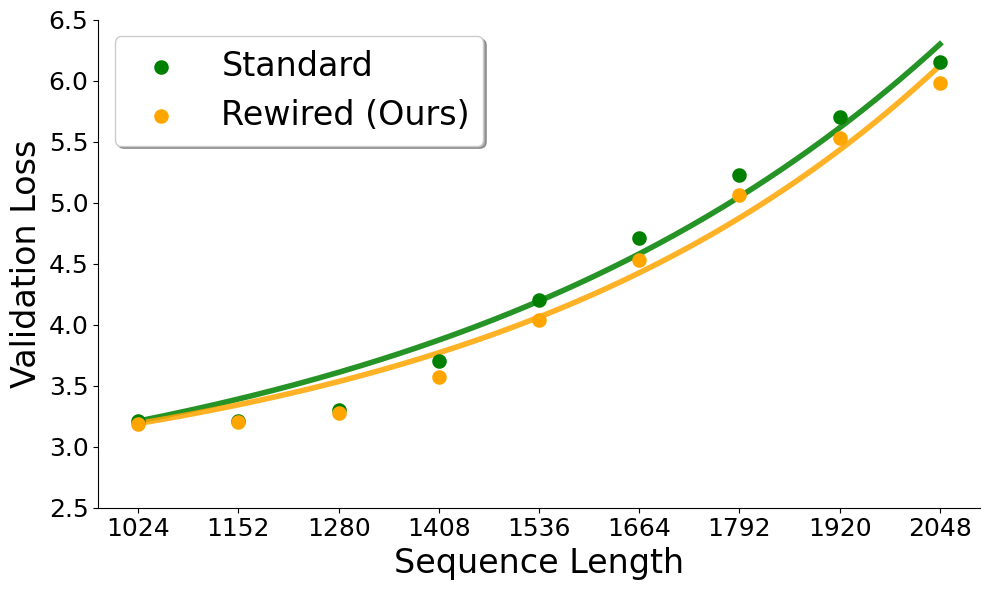}
    \caption{100 million parameter models ($d=12$)}
    \label{fig:extrap_d12}
  \end{subfigure}\hfill
  \begin{subfigure}{0.48\textwidth}
    \centering
    \includegraphics[width=\linewidth]{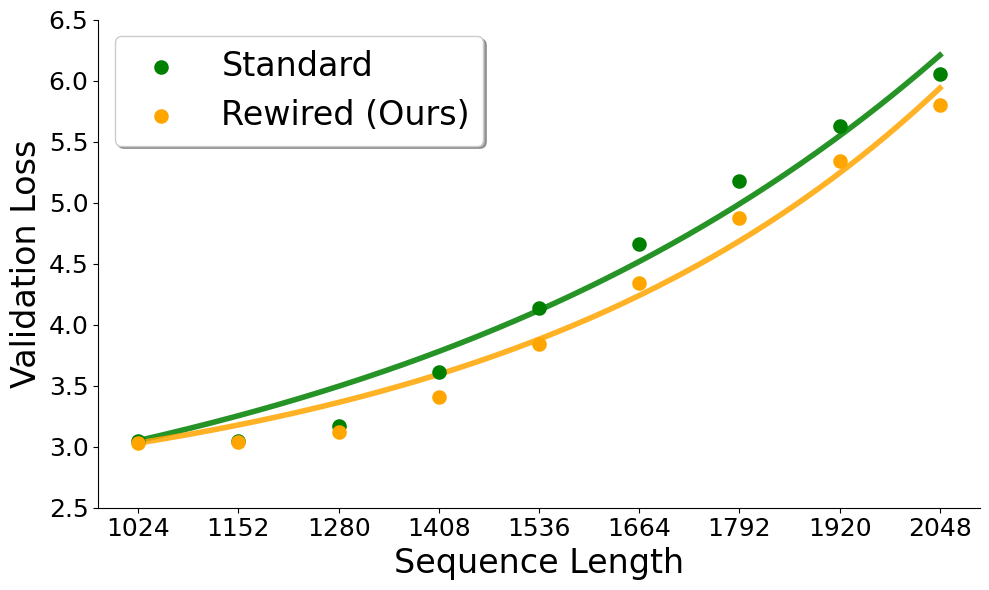}
    \caption{250 million parameter models ($d=16$)}
    \label{fig:extrap_d16}
  \end{subfigure}

  \medskip 

  \begin{subfigure}{0.48\textwidth}
    \centering
    \includegraphics[width=\linewidth]{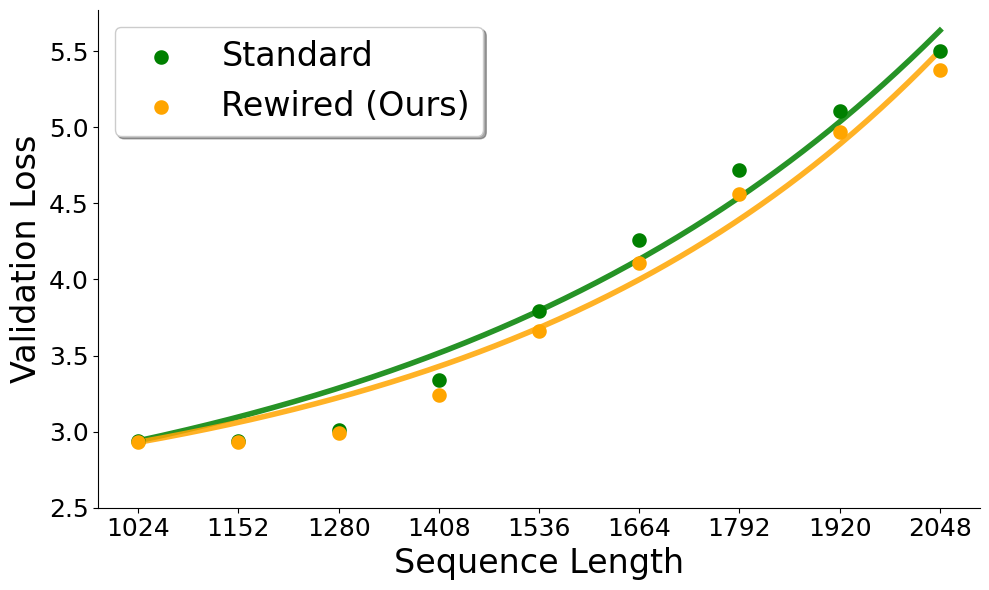}
    \caption{450 million parameter models ($d=20$)}
    \label{fig:extrap_d20}
  \end{subfigure}\hfill
  \begin{subfigure}{0.48\textwidth}
    \centering
    \includegraphics[width=\linewidth]{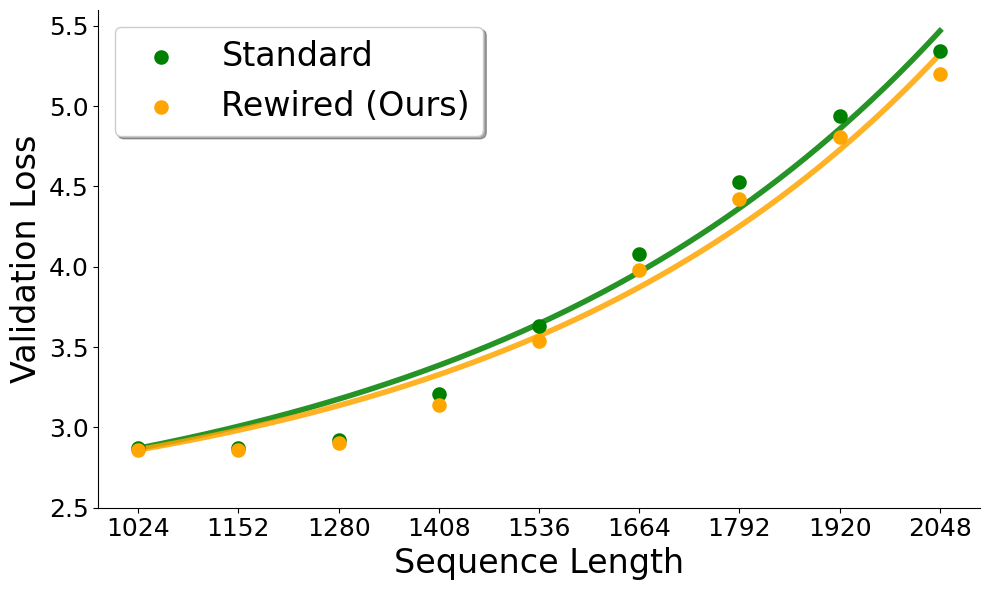}
    \caption{750 million parameter models ($d=24$)}
    \label{fig:extrap_d24}
  \end{subfigure}

  \caption{Extrapolation comparison between standard and rewired transformers at various model sizes: (a) 100 million parameters, (b) 250 million parameters, (c) 450 million parameters, and (d) 750 million parameters.}
  \label{fig:extrap_all4}
\end{figure}

\subsection{Dirichlet Energy Analysis}\label{sec:app_de}
We perform a Dirichlet energy analysis to better understand how our rewiring technique changes the learned information mixing strategy in transformers. Dirichlet energy is a measure of token representation distinctness (sharpness) per layer and is used frequently in GNN literature to validate theoretical claims about representational sharpness. Higher energy means more distinct tokens and lower energy means tokens are more similar to each other. We find in \Cref{fig:de_analysis} that our rewired model maintains consistently higher energy (higher sharpness) through the early-middle layers and converges to stability in the later layers. The final representations in our rewired model actually show lower Dirichlet energy than standard transformer. A more confident convergence in the later layers may suggest that our rewired model can afford more productive mixing in later layers because representations weren't prematurely homogenized earlier. 


\begin{figure}[t]
  \centering
  \includegraphics[width=0.6\linewidth]{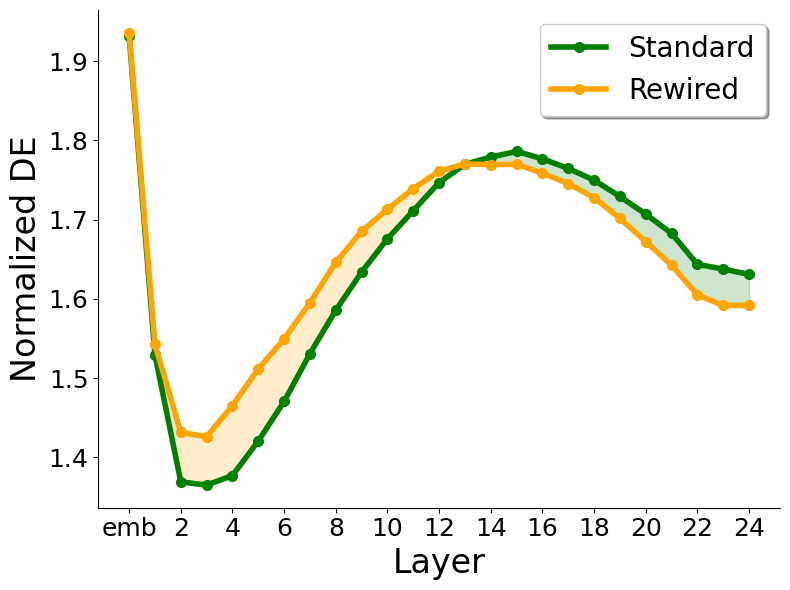}
  \caption{Dirichlet energy comparison between standard and our rewired transformer.}
  \label{fig:de_analysis}
\end{figure}

\subsection{In-Context Long-Range Modeling}\label{sec:app_pg19}
PG-19 dataset consists of full-length books, making it an appropriate testbed for long-context modeling. We find superior long-range performance for our rewired model, with improvements growing with sequence length. 

\begin{table}[h]
  \centering
  \caption{Perplexity comparison between standard and rewired transformers across sequence lengths on PG-19 \cite{rae2019compressivetransformerslongrangesequence}. Lower values indicate better performance.}
  \label{tab:perplexity_extrapolation}
  \begin{tabular}{lcc}
    \toprule
    Sequence Length & Standard Transformer & Rewired (Ours) \\
    \midrule
    1024 (in-context) & 4.2967 & \textbf{4.2626} \\
    2048              & 5.5783 & \textbf{5.3423} \\
    4096              & 7.0703 & \textbf{6.8023} \\
    \bottomrule
  \end{tabular}
\end{table}

\subsection{Throughput and Overhead Analysis}\label{sec:app_overhead}
We provide throughput and overhead information of our rewired transformer compared to standard transformer below in \Cref{tab:overhead}.

\begin{table}[h]
  \centering
  \caption{Throughput and overhead comparison between standard and our rewired transformer.}
  \label{tab:overhead}
  \begin{tabular}{lcc}
    \toprule
    Metric & Standard Transformer & Rewired (Ours) \\
    \midrule
    Training Throughput (tokens/sec) & 88955 & 61039 \\
    Inference Time (sec)             & 22.692   & 25.203 \\
    Inference Throughput (tokens/sec) & 9016.20 & 8117.88 \\
    Inference Memory Overhead (MB)   & 7902.76 & 9028.30 \\
    \bottomrule
  \end{tabular}
\end{table}

\section{Ablation Studies}\label{sec:app_ablation_studies}

\subsection{Dot-Product versus Bilinear Form}
We perform ablation studies on runway coefficient calculation methods. Namely, we compare the parameter-free method of using dot-product (used in our experiments) with the learnable bilinear form. When using the bilinear form, the rewired model introduces additional parameters by design. We show its precise parameter count increase in percentage compared to standard transformer at each model size in \Cref{tab:param_count_bilin} when using the bilinear form. Note that even when using the bilinear form, the parameter-count increase is negligible compared to standard transformers.

\begin{table}[h]
  \centering
  \caption{Model Parameterization and total parameter difference (in percent increase) between the rewired transformer using bilinear form and standard transformer at each model size.}
  \label{tab:param_count_bilin}
  \begin{tabular}{ccccc}
    \toprule
    & \multicolumn{3}{c}{Total Parameter Count} \\
    \cmidrule(r){3-4}
    $d$ & $d_{\text{model}}$ & Std. Transformer & Rewired (Bilinear Form) & Parameter $\Delta$\\
    \midrule
    8 & 512 & 50,927,104 & 50,959,872 & +0.06\%\\
    12 & 768 & 123,607,296 & 123,656,448 & +0.04\%\\
    16 & 1024 & 252,890,112 & 252,955,648 & +0.03\%\\
    20 & 1280 & 457,726,720 & 457,808,640 & +0.02\%\\
    24 & 1536 & 756,969,984 & 757,068,288 & +0.01\%\\
    \bottomrule
  \end{tabular}
\end{table}

\Cref{fig:ablation_main_ppl} shows the main ablation comparing validation loss of the bilinear form rewiring (orange circle) versus the dot-product rewiring (purple diamond) methods. Standard transformer performance (green square) is also included for baseline visualization. We see that the bilinear form and dot-product approaches are comparable in performance. 

\begin{figure}[h]
  \centering
  \includegraphics[width=0.6\linewidth]{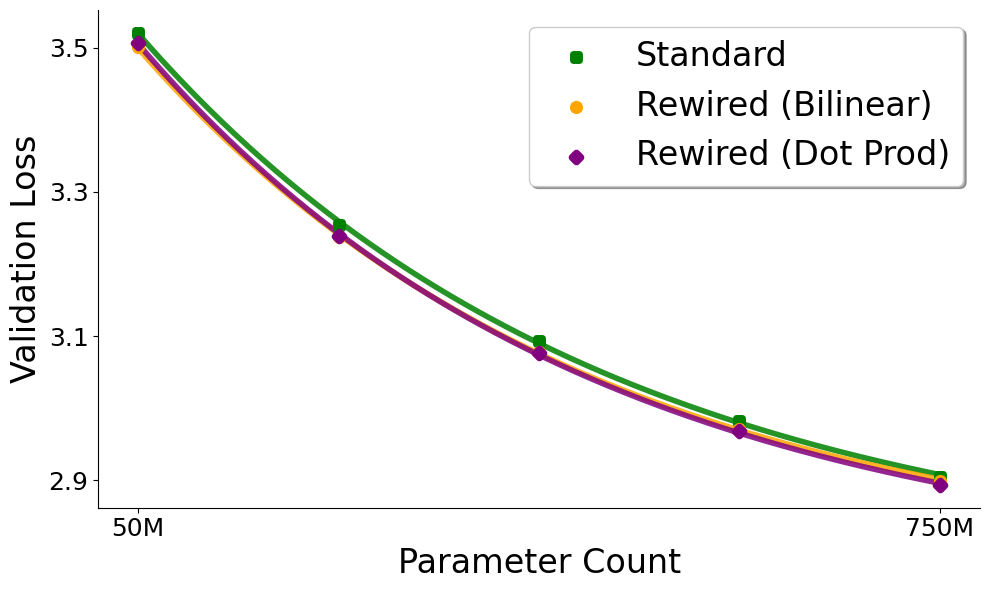}
  \caption{Validation loss comparison at various model sizes between rewired model using dot product (in purple) and bilinear form (in orange). Standard transformer (in green) is included for baseline comparison.}
  \label{fig:ablation_main_ppl}
\end{figure}

\Cref{fig:ablation_seq_len} shows sequence length ablation of both rewired models using dot-product (purple diamond) and bilinear form (orange circle) for runway-coefficient calculation. Standard transformer (in green square) is included for baseline visualization. We see that performance between dot-product and bilinear form methods are comparable, showing that our overall runway-aware rewiring learns approximately equally effective rewiring patterns whether we introduce additional parameters or not.

\begin{figure}[h]
  \centering
  \includegraphics[width=0.6\linewidth]{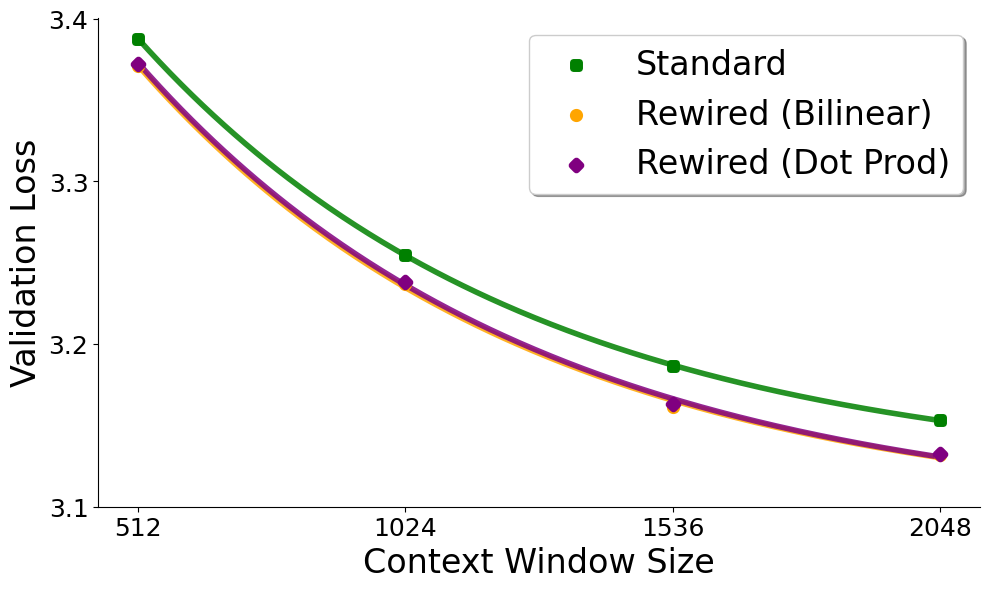}
  \caption{Context window size ablation for rewired transformer with dot product (purple diamond) and bilinear form (orange circle) runway-coefficient calculation. Standard transformer performance (in green square) is also plotted for baseline information.}
  \label{fig:ablation_seq_len}
\end{figure}

\Cref{fig:ablation_extrap} shows extrapolation abilities for transformers with dot-product rewiring compared against bilinear form rewiring. Standard transformer is included for baseline measurements and all models are using 450 million parameters. We see that using dot-product for runway-coefficient calculation is slightly worse than using bilinear form but nonetheless results in superior performance compared to standard transformer. Because this dot-product approach is completely parameter-free, its improvement against standard transformer reinforces the utility of runway-aware rewiring.

\begin{figure}[h]
  \centering
  \includegraphics[width=0.6\linewidth]{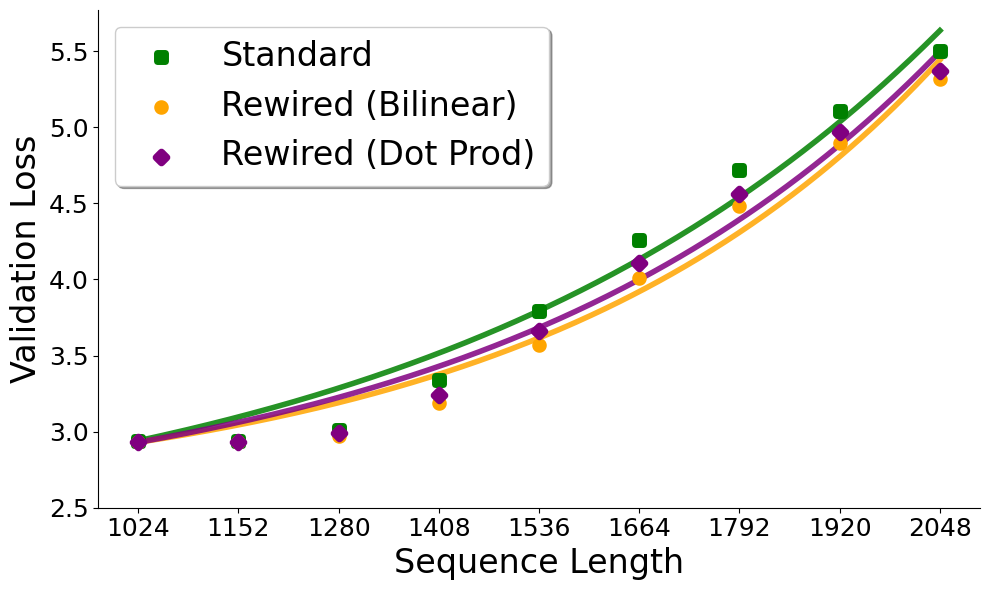}
  \caption{Extrapolation analysis for our rewired transformers using dot product (in purple diamond) and bilinear form (in orange circle). Standard transformer (in green square) is used for baseline visualization. All models are using 450 million parameters.}
  \label{fig:ablation_extrap}
\end{figure}



\subsection{Global Summary Ablation}\label{sec:app_global_summary}
While we use $h_{d-1}$ as the runway signal in our main experiments, we also experiment whether a global summary (such as average of all preceding tokens $h_0,...,h_{d-1}$) may offer a better runway landscape. We find that using a global summary as the runway heuristic results in comparable performance to using $h_{d-1}$ as shown in \Cref{tab:global_summary}.

\begin{table}[h]
  \centering
  \caption{Ablation over rewiring signal choice. Lower values indicate better performance.}
  \label{tab:global_summary}
  \begin{tabular}{lccc}
    \toprule
    Depth & Standard Transformer & Rewired ($h_{d-1}$) & Rewired (Global Summary) \\
    \midrule
    $d=12$ & 3.2547 & \textbf{3.2384} & 3.2397 \\
    $d=24$ & 2.9053 & \textbf{2.8936} & 2.8940 \\
    \bottomrule
  \end{tabular}
\end{table}

\subsection{Uniform Decay Ablation}\label{sec:app_alibi}
we test whether runway-aware rewiring can simply be replaced by uniform, distance-based decay. To do this, we compare our rewired model against standard transformer using ALiBi \cite{press2022trainshorttestlong} as shown in \Cref{tab:alibi}.

\begin{table}[h]
  \centering
  \caption{Comparing the effects of uniform, distance-based decay versus our rewiring technique.}
  \label{tab:alibi}
  \begin{tabular}{lccc}
    \toprule
    Depth & Standard Transformer (ALiBi) & Rewired (ours) \\
    \midrule
    $d=12$ & 3.2798 & \textbf{3.1945} \\
    $d=24$ & 2.9171 & \textbf{2.8936} \\
    \bottomrule
  \end{tabular}
\end{table}

In theory, if our rewiring mechanism simply collapsed to learning distance-based decay (which it theoretically can), standard transformer using pure distance-based decay should be able to match its performance. However, our rewired model (using RoPE) shows consistent improvements compared to standard transformer using ALiBi. This does not show causal evidence that our rewiring technique learns non-trivial patterns in practice. To better investigate this, we perform attention map and rewiring pattern analysis in \Cref{sec:att_pattern_analysis} to visually confirm that our rewiring technique learns dynamic patterns in practice without collapsing to pure distance-based decay.

\section{Rewiring Pattern Analysis}\label{sec:att_pattern_analysis}
Here we conduct analysis over the rewiring behaviors and the overall attention patterns in our 450 million parameter rewired transformer. 

\Cref{fig:att_map_main} visualizes the strength of $\overline{r}$ in the overall attention map. Purple colors denote stronger down-scaling and thus more aggressive rewiring and white and green colors denote weaker down-scaling. Overall, we see that token pairs with longer runways (near the bottom left corner of the matrix) get down-scaled heavily while tokens with shorter runways (near the diagonal) are relatively unaffected by rewiring. Further more, we see that even though token pairs with longer runways typically see stronger down-scaling, the effect is not necessarily uniform and there are many instances where these pairs avoid strong down-scaling (shown in neutral white colors that form a checkered pattern among the purple regions in the bottom left corner). This shows that our runway-aware rewiring is not learning trivial, generalized patterns such as uniform distance-based decay, but rather it can learn to be flexible and choose complex rewiring patterns.

\begin{figure}[h]
  \centering
  \includegraphics[width=0.5\linewidth]{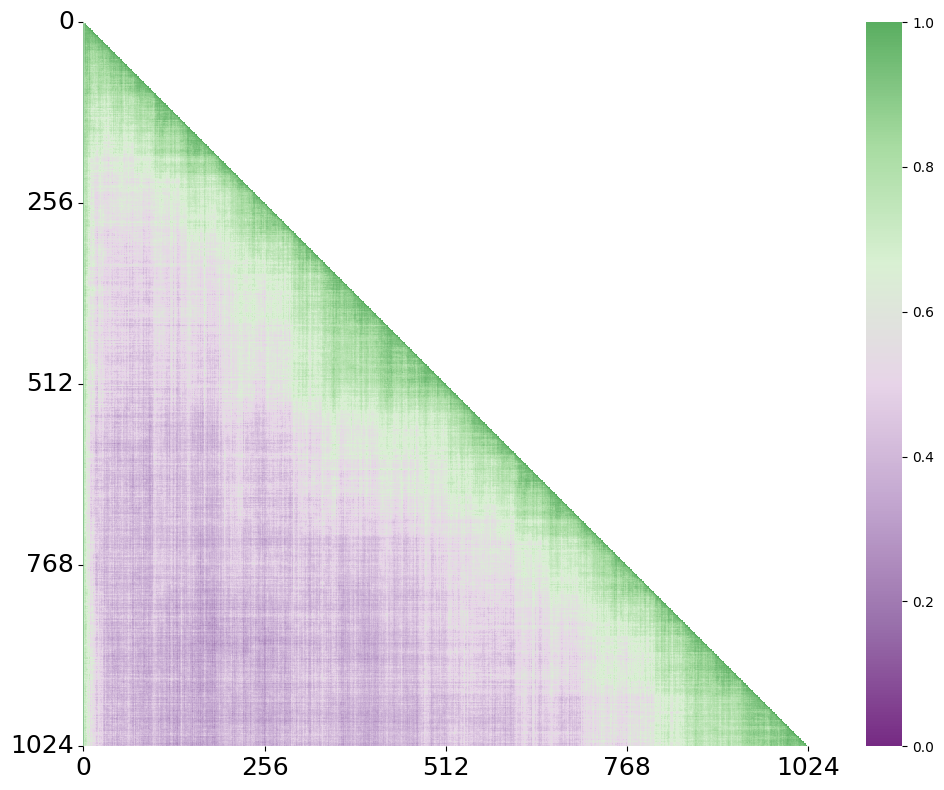}
  \caption{Visualization of the strength of edge down-scaling factor $\overline{r}$ in the overall lower-triangular attention map (averaged over all heads). Lower values (which are more purple) indicate stronger down-scaling and therefore more aggressive rewiring.}
  \label{fig:att_map_main}
\end{figure}

\begin{figure*}[h]
  \centering
  \begin{subfigure}[t]{0.48\textwidth}
    \centering
    \includegraphics[width=\linewidth]{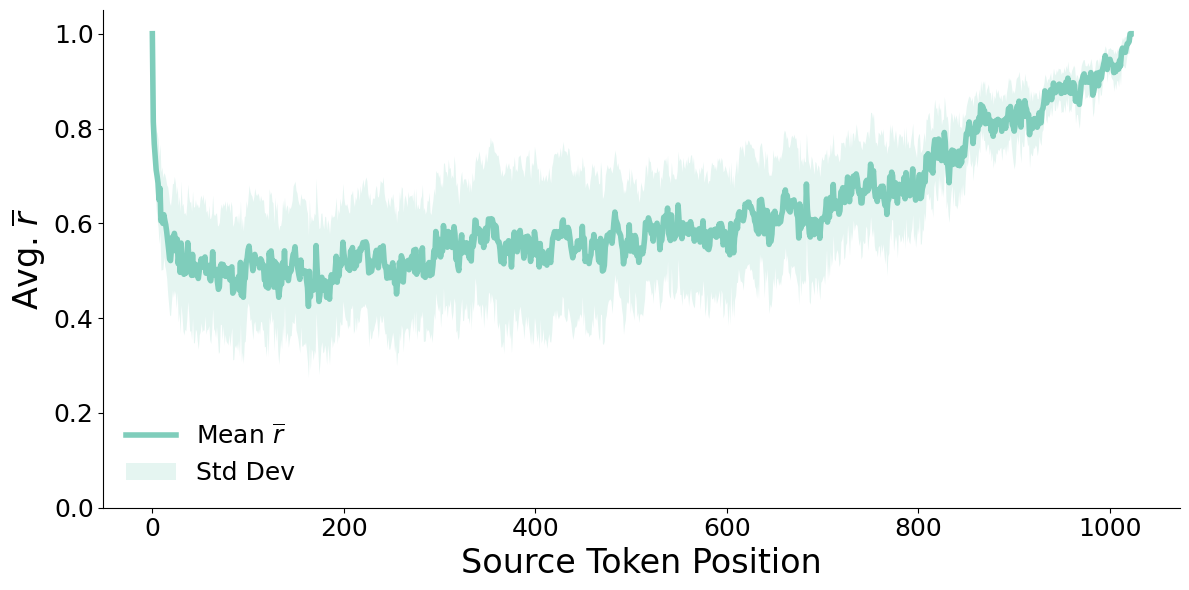}
    \caption{}
    \label{fig:avg_source}
  \end{subfigure}\hfill
  \begin{subfigure}[t]{0.48\textwidth}
    \centering
    \includegraphics[width=\linewidth]{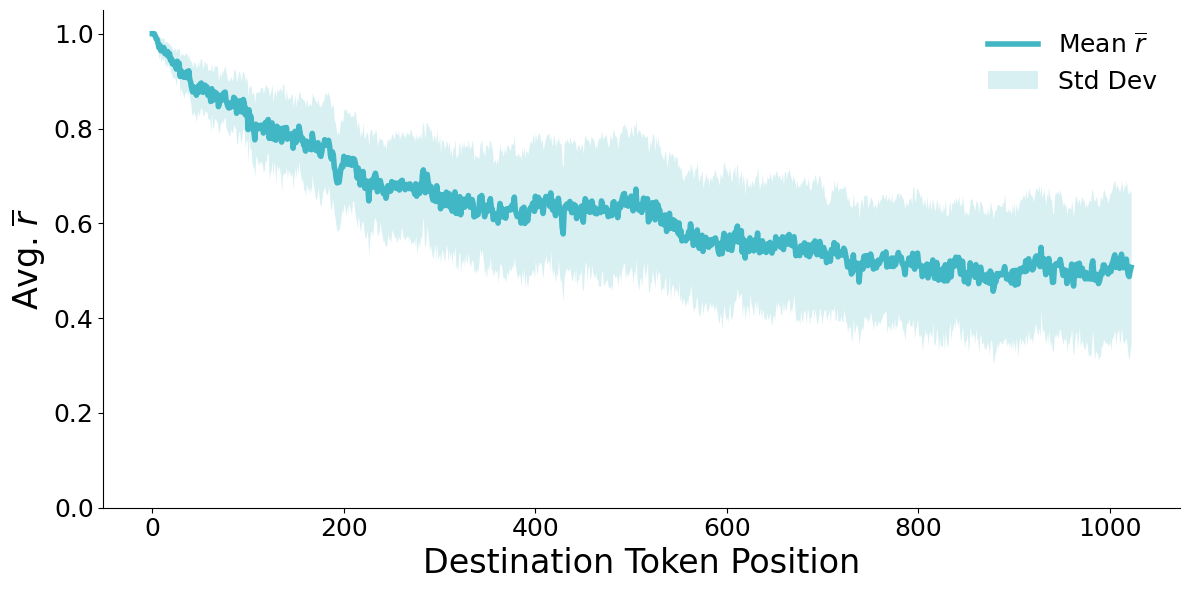}
    \caption{}
    \label{fig:avg_destination}
  \end{subfigure}

  \caption{The average down-scaling effect of $\overline{r}$ with respect to each (a) source token position and (b) destination token position.}
  \label{fig:avg_src_dest}
\end{figure*}

\Cref{fig:avg_src_dest} chart average $\overline{r}$ applied (a) to each source token (\Cref{fig:avg_source}) and (b) by each destination token \Cref{fig:avg_destination}. Lower values indicate stronger down-scaling and thus more aggresive rewiring. The shaded regions illustrate standard deviation and signify diversity of $\overline{r}$ applied to/by each token.

\Cref{fig:avg_source} charts the average $\overline{r}$ applied to each source token by all future tokens. As such, we see an increasing trend across source token positions, with early tokens being down-scaled more heavily than later tokens on average. The shaded region is the standard deviation and shows the diversity of $\overline{r}$ applied to each source token. We see that the standard deviations are quite high for the initial two-thirds of source tokens in the sequence, showing that different future tokens apply different scaling factors to them. This again shows that the rewiring mechanism is not simply down-scaling early tokens in a uniform, distance-based decay but is showing more complex patterns where different future tokens apply noticeably different down-scaling for earlier tokens, even though on average, earlier tokens get down-scaled more as expected. The first few source tokens are not down-scaled as a result of our rewiring methodology and thus shows a spike at the very left end of the chart.

\Cref{fig:avg_destination} charts the average $\overline{r}$ for each destination token's row-wise attention pattern. In other words, this chart represents a horizontal slice through the rewiring pattern matrix, averaging across columns for each row. Thus, this chart shows how much each token down-scales its own attention on average. We see a clean downward pattern across destination token positions, with growing standard deviations. Generally, this means that destination tokens toward the tail-end of the sequence apply stronger and more diverse rewiring, which makes sense given that later tokens have to work with longer runways.

\clearpage
\begin{lstlisting}[float, floatplacement=tp, language=Python, caption={Attention with Runway Aware Rewiring (Pseudo code)},label={listing:rewiring}]
def runway_rewired_attention(x, is_causal=True):
    Q, K, V = project_qkv(x)  # each: [batch, n_heads, seq_len, head_dim]
    
    Q, K = RoPE(Q, K)

    # standard attention logic
    logits = (Q @ K.T) / sqrt(head_dim)
    
    i = torch.arange(seq_len).view(1,1,seq_len,1)
    j = torch.arange(seq_len).view(1,1,1,seq_len)
    
    causal_mask = (j <= i)                              # standard causality
    keep_adjacent = (j == i) | (j == i-1) | (j == 0)    # keep self, prev, first token
    eligible_for_rewiring = (j <= i-2)                   
    
    # extract last head to re-purpose for runway coefficients
    V_last = V[:, -1:, :, :]  # [batch, 1, seq_len, head_dim]
    
    # get V_{i-1} for each position i
    V_prev = V_last[:, :, prev_indices, :]  # [batch, 1, seq_len, head_dim]
    
    # runway-coefficient calculation
    R_scores = v_prev @ v_last.T
    R_scores = R_scores / sqrt(self.head_dim)
    
    R = sigmoid(R_scores)  # runway-coefficient in (0, 1)
    
    beta = 1 - R  # higher beta = lower R
    
    # Override: always keep adjacent edges (no rewiring)
    beta = where(keep_adjacent, 1.0, beta)
    
    # Only apply rewiring to eligible edges
    beta = where(eligible_for_rewiring, beta, 1.0)
    
    # Broadcast to all heads
    beta = beta.expand(batch, n_heads, seq_len, seq_len)
    
    # standard attention logic
    masked_logits = logits.masked_fill(~causal_mask, -inf)
    attn = softmax(masked_logits, dim=-1)
    
    # applying rewiring
    attn_rewired = attn * beta

    # re-normalizing attn_rewired to sum up to 1.0 in each row
    attn_rewired = attn_rewired / attn_rewired.sum(dim=-1)
    
    # standard attention logic
    out = attn_rewired @ V  # [batch, n_heads, seq_len, head_dim]
    out = concat_heads(out)   # [batch, seq_len, d_model]
    out = project_out(out)
    
    return out
\end{lstlisting}

\clearpage

\newpage
\section*{NeurIPS Paper Checklist}

\begin{enumerate}

\item {\bf Claims}
    \item[] Question: Do the main claims made in the abstract and introduction accurately reflect the paper's contributions and scope?
    \item[] Answer: \answerYes{} 
    \item[] Justification: all claims made are supported by either theoretical analysis in section 3 or empirical findings in section 5 and appendix.

\item {\bf Limitations}
    \item[] Question: Does the paper discuss the limitations of the work performed by the authors?
    \item[] Answer: \answerYes{} 
    \item[] Justification: for example, we discuss why our rewiring might underperform on tasks requiring information co-occurrence.

\item {\bf Theory assumptions and proofs}
    \item[] Question: For each theoretical result, does the paper provide the full set of assumptions and a complete (and correct) proof?
    \item[] Answer: \answerYes{} 
    \item[] Justification: see appendix A.

    \item {\bf Experimental result reproducibility}
    \item[] Question: Does the paper fully disclose all the information needed to reproduce the main experimental results of the paper to the extent that it affects the main claims and/or conclusions of the paper (regardless of whether the code and data are provided or not)?
    \item[] Answer: \answerYes{} 
    \item[] Justification: we include all code in our submission and make our implementation details clear in the appendix.

\item {\bf Open access to data and code}
    \item[] Question: Does the paper provide open access to the data and code, with sufficient instructions to faithfully reproduce the main experimental results, as described in supplemental material?
    \item[] Answer: \answerYes{} 
    \item[] Justification: we include our code in our submission. All data used are cited clearly.

\item {\bf Experimental setting/details}
    \item[] Question: Does the paper specify all the training and test details (e.g., data splits, hyperparameters, how they were chosen, type of optimizer) necessary to understand the results?
    \item[] Answer: \answerYes{} 
    \item[] Justification: Refer to the appendix.

\item {\bf Experiment statistical significance}
    \item[] Question: Does the paper report error bars suitably and correctly defined or other appropriate information about the statistical significance of the experiments?
    \item[] Answer: \answerYes{} 
    \item[] Justification: We provide std. dev. where appropriate.

\item {\bf Experiments compute resources}
    \item[] Question: For each experiment, does the paper provide sufficient information on the computer resources (type of compute workers, memory, time of execution) needed to reproduce the experiments?
    \item[] Answer: \answerYes{} 
    \item[] Justification: We provide compute resources used for training in section 5.
    
\item {\bf Code of ethics}
    \item[] Question: Does the research conducted in the paper conform, in every respect, with the NeurIPS Code of Ethics \url{https://neurips.cc/public/EthicsGuidelines}?
    \item[] Answer: \answerYes{} 
    \item[] Justification: We conform to NeurIPS code of ethics.

\item {\bf Broader impacts}
    \item[] Question: Does the paper discuss both potential positive societal impacts and negative societal impacts of the work performed?
    \item[] Answer: \answerYes{} 
    \item[] Justification: We provide a broader impact statement in the appendix.
    
\item {\bf Safeguards}
    \item[] Question: Does the paper describe safeguards that have been put in place for responsible release of data or models that have a high risk for misuse (e.g., pre-trained language models, image generators, or scraped datasets)?
    \item[] Answer: \answerNA{} 
    \item[] Justification: not applicable.

\item {\bf Licenses for existing assets}
    \item[] Question: Are the creators or original owners of assets (e.g., code, data, models), used in the paper, properly credited and are the license and terms of use explicitly mentioned and properly respected?
    \item[] Answer: \answerYes{} 
    \item[] Justification: We properly cite and credit authors for any datasets that we use. All code and models were created by the authors of this paper.

\item {\bf New assets}
    \item[] Question: Are new assets introduced in the paper well documented and is the documentation provided alongside the assets?
    \item[] Answer: \answerYes{} 
    \item[] Justification: any new asset introduced in the paper are well documented.

\item {\bf Crowdsourcing and research with human subjects}
    \item[] Question: For crowdsourcing experiments and research with human subjects, does the paper include the full text of instructions given to participants and screenshots, if applicable, as well as details about compensation (if any)? 
    \item[] Answer: \answerNA{} 
    \item[] Justification: not applicable.

\item {\bf Institutional review board (IRB) approvals or equivalent for research with human subjects}
    \item[] Question: Does the paper describe potential risks incurred by study participants, whether such risks were disclosed to the subjects, and whether Institutional Review Board (IRB) approvals (or an equivalent approval/review based on the requirements of your country or institution) were obtained?
    \item[] Answer: \answerNA{} 
    \item[] Justification: NA

\item {\bf Declaration of LLM usage}
    \item[] Question: Does the paper describe the usage of LLMs if it is an important, original, or non-standard component of the core methods in this research? Note that if the LLM is used only for writing, editing, or formatting purposes and does \emph{not} impact the core methodology, scientific rigor, or originality of the research, declaration is not required.
    \item[] Answer: \answerNA{} 
    \item[] Justification: We did not use LLMs for important, original, or non-standard component of the core methods in this research but we did use LLMs for help in editing our manuscript.

\end{enumerate}

\end{document}